%% file: master.tex
\relax

\documentclass{article}
\usepackage{times, microtype, fullpage}

\usepackage{hyperref}
\usepackage{url}
\usepackage{mathtools}
\usepackage{amsfonts, amssymb, amsthm}
\usepackage[algo2e,boxed]{algorithm2e}
\usepackage{enumitem}
\usepackage{thmtools}
\usepackage{graphicx}
\usepackage{xcolor}
\usepackage{natbib}
\usepackage{framed}
\usepackage{caption, subcaption}
\usepackage{newfloat}

\DeclareFloatingEnvironment[
    listname=Algorithm,
    name=Algorithm,
    placement=tbhp
]{myalg}

\title{Label Efficient Learning by Exploiting Multi-class Output Codes}
\author{
  Maria Florina Balcan\\
  School of Computer Science\\
  Carnegie Mellon University\\
  \texttt{ninamf@cs.cmu.edu}
  \and
  Travis Dick\\
  School of Computer Science\\
  Carnegie Mellon University\\
  \texttt{tdick@cs.cmu.edu}
  \and
  Yishay Mansour\\
  Microsoft Research and\\
  Tel Aviv University\\
  \texttt{mansour@tau.ac.il}
}
\date{}
\input{commands.tex}

\begin{document}

\maketitle

\input{abstract.tex}
\input{introduction.tex}
\input{related_work.tex}
\input{preliminaries.tex}
\input{hamming_dp1.tex}
\input{one_vs_all.tex}
\input{boundary_features.tex}

\input{agnostic.tex}
\input{discussion.tex}
\input{acknowledgements.tex}

\clearpage
\bibliographystyle{plainnat}
\bibliography{references}

\newpage

\input{hamming_dp1_appendix.tex}
\input{one_vs_all_appendix.tex}

\input{boundary_features_appendix.tex}

\end{document}

%% file: commands.tex
\newtheorem{thm}{Theorem}
\newtheorem*{thm*}{Theorem}
\newtheorem{lem}{Lemma}

\newtheorem{defn}{Definition}
\newtheorem{assm}{Assumption}

\renewcommand \paragraph [1]{\noindent{\bf #1}}

\newcommand \reals {\mathbb{R}}

\newcommand \vol {\operatorname{Vol}}

\newcommand \dist {d}
\newcommand \ball {B}
\newcommand \halfball {B^{1/2}}
\newcommand \hbp {p^{1/2}}

\newcommand \unitsphere {\mathcal{S}^{d-1}}
\newcommand \sphere {\unitsphere}
\newcommand \sam {\mu_{\circ}}
\newcommand \ang {\theta}
\newcommand \scap {C}
\newcommand \scapvol {V^d}

\newcommand \cone {\operatorname{cone}}

\newcommand \prob {\operatorname*{Pr}}

\newcommand \argmin {\operatorname*{argmin}}

\newcommand \hamd {d_{\rm Ham}}

\newcommand \sign {\operatorname{sign}}
\newcommand \error {\operatorname{err}_{\datadist}}

\newcommand \datadist {P}

\newcommand \K {K} 
\newcommand \truelabel {f^*} 
\newcommand \X {\mathcal{X}} 
\newcommand \C {C} 
 
\newcommand \numclasses {L} 
\newcommand \numseparators {m} 
\newcommand \sampleset {S}
\newcommand \bmin {b_{\rm min}}
\newcommand \clb {c_{\rm lb}}
\newcommand \cub {c_{\rm ub}}
\newcommand \qlb [1]{q^{(#1)}_{\rm lb}}
\newcommand \qub [1]{q^{(#1)}_{\rm ub}}
\newcommand \ql {q_{\rm lb}}
\newcommand \qu {q_{\rm ub}}
\newcommand \rlb [1]{\rho^{(#1)}_{\rm lb}}
\newcommand \rub [1]{\rho^{(#1)}_{\rm ub}}

\newcommand \interior [2]{\operatorname{int}_{#2}(#1)}

\newcommand \setc [2]{\{#1 \,:\, #2\}}

\newcommand \norm [1]{\Vert #1 \Vert}

\newcommand \tp {\top}

\newcommand \ind[1]{\mathbb{I}\{#1\}}

%% file: abstract.tex
\begin{abstract}
  We present a new perspective on the popular multi-class algorithmic
  techniques of one-vs-all and error correcting output codes. Rather
  than studying the behavior of these techniques for supervised
  learning, we establish a connection between the success of these
  methods and the existence of label-efficient learning procedures. We
  show that in both the realizable and agnostic cases, if output codes
  are successful at learning from labeled data, they implicitly assume
  structure on how the classes are related. By making that structure
  explicit, we design learning algorithms to recover the classes with
  low label complexity. We provide results for the commonly studied
  cases of one-vs-all learning and when the codewords of the classes
  are well separated. We additionally consider the more challenging
  case where the codewords are not well separated, but satisfy a
  boundary features condition that captures the natural intuition that
  every bit of the codewords should be significant.
\end{abstract}

%% file: introduction.tex
\section{Introduction}

\paragraph{Motivation:} Large scale multi-class learning problems with
an abundance of unlabeled data are ubiquitous in modern machine
learning. For example, an in-home assistive robot needs to learn to
recognize common household objects, familiar faces, facial
expressions, gestures, and so on in order to be useful. Such a robot
can acquire large amounts of unlabeled training data simply by
observing its surroundings, but it would be prohibitively time
consuming (and frustrating) to ask its owner to annotate any
significant portion of this raw data. More generally, in many modern
learning problems we often have easy and cheap access to large
quantities of unlabeled training data (e.g., on the internet) but
obtaining high-quality labeled examples is relatively expensive.  More
examples include text understanding, recommendation systems, or
wearable computing~\citep{Thrun96b, ThrunM95, ThrunM95b,
  NELL-aaai15}. The scarcity of labeled data is especially pronounced
in problems with many classes, since supervised learning algorithms
typically require labeled examples from every class. In such settings,
algorithms should strive to make the best use of unlabeled data in
order to minimize the need for expensive labeled examples.

\paragraph{Overview:} We approach label-efficient learning by making
the implicit assumptions of popular multi-class learning algorithms
explicit and showing that they can also be exploited when learning
from limited labeled data. We focus on a family of techniques called
\emph{output codes} that work by decomposing a given multi-class
problem into a collection of binary classification tasks
\citep{mehmohri,ECOC95,ajcolt05,ajalt09}. The novelty of our results
is to show that the existence of various low-error output codes
constrains the distribution of unlabeled data in ways that can be
exploited to reduce the label complexity of learning. We consider both
the consistent setting, where the output code achieves zero error, and
the agnostic setting, where the goal is to compete with the best
output code. The most well known output code technique is one-vs-all
learning, where we learn one binary classifier for distinguishing each
class from the union of the rest. When output codes are successful at
learning from labeled data, it often implies geometric structure in
the underlying problem. For example, if it is possible to learn an
accurate one-vs-all classifier with linear separators, it implies that
no three classes can be collinear, since then it would be impossible
for a single linear separator to distinguish the middle class from the
union of the others. In this work exploit this implicitly assumed
structure to design label-efficient algorithms for the commonly
assumed cases of one-vs-all and error correcting output codes, as well
as a novel boundary features condition that captures the intuition
that every bit of the codewords should be significant.

\paragraph{Our results:} Before discussing our results, we briefly
review the output code methodology. For a problem with $L$ classes, a
domain expert designs a code matrix $C \in \{\pm 1\}^{L \times m}$
where each column partitions the classes into two meaningful
groups. The number of columns $m$ is chosen by the domain expert. For
example, when recognizing household objects we could use the following
true/false questions to define the partitions: ``is it made of
wood?'', ``is it sharp?'', ``does it have legs?'', ``should I sit on
it?'', and so on. Each row of the code matrix describes one of the
classes in terms of these partitions (or semantic features). For
example, the class ``table'' could be described by the vector
$(+1,-1,+1,-1)$, which is called the class' codeword. Once the code
matrix has been designed, we train an output code by learning a binary
classifier for each of the binary partitions (e.g., predicting whether
an object is made of wood or not). To predict the class of a new
example, we predict its codeword in $\{\pm 1\}^m$ and output the class
with the nearest codeword under the Hamming distance. Two popular
special cases of output codes are one-vs-all learning, where $C$ is
the identity matrix (with -1 in the off-diagonal entries), and error
correcting output codes, where the Hamming distance between the
codewords is large.

In each of our results we assume that there exists a consistent or
low-error linear output code classifier and we impose constraints on
the code matrix and the distribution that generates the data. We
present algorithms and analysis techniques for a wide range of
different conditions on the code matrix and data distribution to
showcase the variety of implicit structures that can be exploited. For
the code matrix, we consider the case when the codewords are well
separated (i.e., the output code is error correcting), the case of
one-vs-all (where the code matrix is the identity), and a natural
boundary features condition. These conditions can loosely be compared
in terms of the Hamming distance between codewords. In the case of
error correcting output codes, the distance between codewords is large
(at least $d+1$ when the data is $d$-dimensional), in one-vs-all the
distance is always exactly $2$, and finally in the boundary features
condition the distance can be as small as $1$. In the latter cases,
the lower Hamming distance requirement is balanced by other structure
in the code matrix. For the distribution, we either assume that the
data density function satisfies a thick level set condition or that
the density is upper and lower bounded on its support. Both regularity
conditions are used to ensure that the geometric structure implied by
the consistent output code will be recoverable based on a sample of
data.

\paragraph{Error correcting output codes:} We first showcase how to
exploit the implicit structure assumed by the commonly used and
natural case of linear output codes where the Hamming distance between
codewords is large. In practice, output codes are designed to have
this property in order to be robust to prediction errors for the
binary classification tasks \citep{ECOC95}. We suppose that the output
code makes at most $\beta$ errors when predicting codewords and has
codewords with Hamming distance at least $2\beta + d + 1$ in a
$d$-dimensional problem. The key insight is that when the code words
are well separated, this implies that points belonging to different
classes must be geometrically separated as well. This suggests that
tight clusters of data will be label-homogeneous, so we should be able
to learn an accurate classifier using only a small number of label
queries per cluster. The main technical challenge is to show that our
clustering algorithm will not produce too many clusters (in order to
keep the label complexity controlled), and that with high probability,
a new sample from the distribution will have the same label as its
nearest cluster. We show that when the data density satisfies a
thick-level set condition (requiring that its level sets do not have
bridges or cusps that are too thin), then a single-linkage clustering
algorithm can be used to recover a small number of label-homogeneous
clusters.

\paragraph{One-vs-all:} Next, we consider the classic one-vs-all
setting for data in the unit ball. This is an interesting setting
because of the popularity of one-vs-all classification and because it
significantly relaxes the assumption that the codewords are well
separated (in a one-vs-all classifier, the Hamming distance between
codewords is exactly 2). The main challenge in this setting is that
there need not be a margin between classes and a simple single-linkage
style clustering might group multiple classes into the same
cluster. To overcome this challenge, we show that the classes are
probabilistically separated in the following sense: after projecting
onto the surface of the unit ball, the level sets of the projected
density are label-homogeneous. Equivalently, the high-density regions
belonging to different classes must be separated by low-density
regions. We exploit this structure by estimating the connected
components of the $\epsilon$ level set using a robust single-linkage
clustering algorithm.

\paragraph{The boundary features condition:} Finally, we introduce an
interesting and natural condition on the code matrix capturing the
intuition that every binary learning task should be significant. This
condition has the weakest separation requirement, allowing the
codewords to have a Hamming distance of only 1. This setting is our
most challenging, since it allows for the classes to be very well
connected to one another, which prevents clustering or level set
estimation from being used to find a small number of label-homogeneous
clusters. Nevertheless, we show that the implicit geometric structure
implied by the output code can be exploited to learn using a small
number of label queries. In this case, rather than clustering the
unlabeled sample, we apply a novel hyperplane-detection algorithm that
uses the \emph{absence} of data to learn local information about the
boundaries between classes. We then use the implicit structure of the
output code to extend these local boundaries into a globally accurate
prediction rule.

\paragraph{Agnostic Setting:} Finally, we show that our results for
the error correcting, one-vs-all, and boundary features cases can all
be extended to an agnostic learning setting, where we do not assume
that there exists a consistent output code classifier.

Our results show an interesting trend: when linear output codes are
able to learn from labeled data, it is possible to exploit the same
underlying structure in the problem to learn using a small number of
label requests. Our results hold under several natural assumptions on
the output code and general conditions on the data distribution, and
employ both clustering and hyperplane detection strategies to reduce
the label complexity of learning.

%% file: related_work.tex
\section{Related Work}

Reduction to binary classification is one of the most widely used
techniques in applied machine learning for attacking multi-class
problems. Indeed, the one-vs-all, one-vs-one, and the error correcting
output code approaches~\citep{ECOC95} all follow this structure
\citep{mehmohri,ajcolt05,ajalt09,Daniely2012,Allwein2000}.

There is no prior work providing error bounds for output codes using
unlabeled data and interaction.  There has been a long line of work
for providing provable bounds for semi-supervised
learning~\citep{BBY05,BB10, Blum1998, Chapelle2014} and active
learning~\citep{BBL06, D11, BU15, Hanneke2014}. These works provide
bounds on the benefits of unlabeled data and interaction for
significantly different semi-supervised and active learning methods
that are based different assumptions, often focusing on binary
classification, thus the results are largely incomparable. Another
line of recent work considers the multi-class setting and uses
unlabeled data to consistently estimate the risk of classifiers when
the data is generated from a known family of models \citep{Donmez2010,
  Donmez2011a, Donmez2011b}. Their results do not immediately imply
learning algorithms and they consider generative assumptions, while in
contrast our work explicitly designs learning algorithms under
commonly used discriminative assumptions.

Another work related to ours is that of \citet{BBM13}, where labels
are recovered from unlabeled data. The main tool that they use, in
order to recover the labels, is the assumption that there are multiple
views and an underlying ontology that are known, and restrict the
possible labeling. Similarly, \citet{Steinhardt2016} show how to use
the method of moments to estimate the risk of a model from unlabeled
data under the assumption that the data has three independent
views. Our work is more widely applicable, since it applies when we
have only a single view.

The output-code formalism is also used by~\citet{zero-shotpmhm09} for
the purpose of zero shot learning. They demonstrate that it is
possible to exploit the semantic relationships encoded in the code
matrix to learn a classifier from labeled data that can predict
accurately even classes that \emph{did not appear in the training
  set}. These techniques make very similar assumptions to our work but
require that the code matrix $C$ is known and the problem that they
solve is different.

%% file: preliminaries.tex
\section{Preliminaries}
\label{sec:preliminaries}

We consider multiclass learning problems over an instance space
$\X \subset \reals^d$ where each point is labeled by
$\truelabel : \X \to \{1, \dots, \numclasses\}$ to one out of
$\numclasses$ classes and the probability of observing each outcome
$x \in \X$ is determined by a data distribution $\datadist$ on
$\X$. The density function of $P$ is denoted by
$p : \X \to [0,\infty)$. In all of our results we assume that there
exists a consistent (but unknown) linear output-code classifier
defined by a code matrix $C \in \{\pm 1\}^{L \times m}$ and $m$ linear
separators $h_1$, \dots, $h_m$. We denote class $i$'s code word by
$C_i$ and define $h(x)~=~(\sign(h_1(x)), \dots, \sign(h_m(x)))$ to be
the predicted code word for point $x$. We let $\hamd(c,c')$ denote the
Hamming distance between any codewords $c,c' \in \{\pm
1\}^m$. Finally, to simplify notation, we assume that the diameter of
$\X$ is at most 1.

Our goal is to learn a hypothesis
$\hat f : \X \to \{1, \dots, \numclasses\}$ minimizing
$\error(\hat f) = \prob_{X \sim P}(\hat f(x) \neq f(x))$ from an
unlabeled sample drawn from the data distribution $\datadist$ together
with a small set of actively queried labeled examples.

Finally, we use the following notation throughout the paper: For any
set $A$ in a metric space $(\X, d)$, the $\sigma$-interior of $A$ is
the set
$\interior{A}{\sigma} = \setc{x \in A}{B(x,\sigma) \subset A}$. The
notation $\tilde O(\cdot)$ suppresses logarithmic terms.

%% file: hamming_dp1.tex
\section{Error Correcting Output Codes}
\label{sec:hamdp1}

We first consider the implicit structure when there exists a
consistent linear \emph{error correcting} output code classifier:
\begin{assm}
  \label{assm:ecoc}
  There exists a code matrix $C \in \{\pm 1\}^{L \times m}$ and linear
  functions $h_1$, \dots, $h_m$ such that: (1) there exists
  $\beta \geq 0$ such that any point $x$ from class $y$ satisfies
  $\hamd(h(x), C_y) \leq \beta$, (2) The Hamming distance between the
  codewords of $C$ is at least $2\beta + d + 1$; and (3) at most $d$
  of the separators $h_1$, \dots, $h_m$ intersect at any point.
\end{assm}

Part (1) of this condition is a bound on the number of linear
separators that can make a mistake when the output code predicts the
codeword of a new example, part (2) formalizes the requirement of
having well separated codewords, and part (3) requires that the
hyperplanes be in general position, which is a very mild condition
that can be satisfied by adding an arbitrarily small perturbation to
the linear separators.

Despite being very natural, Assumption~\ref{assm:ecoc} conveniently
implies that there exists a distance $g > 0$ such that any points that
$f^*$ assigns to different classes must be at least distance $g$
apart. To see this, fix any pair of points $x$ and $x'$ with
$f^*(x) \neq f^*(x')$. By the triangle inequality, we have that
$\hamd(h(x), h(x')) \geq d+1$, implying that the line segment $[x,x']$
crosses at least $d+1$ of the linear separators. Since only $d$ linear
separators can intersect at a point, the line segment must have
non-zero length. Applying this argument to the closest pair of points
between all pairs of classes and taking the minimum length gives the
result. A formal proof is given Section~\ref{apdx:hamdp1} of the
Appendix.

\begin{restatable}{lem}{SeparationLemma}
  \label{lem:separation}
  Under Assumption~\ref{assm:ecoc}, there exists $g > 0$ s.t. if
  points $x$ and $x'$ belong to different classes, then
  $\norm{x - x'} > g$.
\end{restatable}

Lemma~\ref{lem:separation} suggests that we should be able to reduce
the label complexity of learning by clustering the data and querying
the label of each cluster, since nearby points must belong to the same
class. If we use a single-linkage style clustering algorithm that
merges clusters whenever their distance is smaller than $g$, we are
guaranteed that the clusters will be label-homogeneous, and therefore
we can recover nearly all of the labels by querying one label from the
largest clusters. See Algorithm~\ref{alg:singlelinkage} for
pseudocode.

\begin{myalg}[htb]
  \begin{framed}
  \hangindent=0.7cm {\bf Input:} Sample $\sampleset = \{x_1, \dots, x_n\}$, 
  radius $r_c > 0$, target error $\epsilon > 0$
  \begin{enumerate}[noitemsep,nolistsep,leftmargin=*]
  \item Let $\{\hat A_1\}_{i=1}^N$ be the connected components of the
    graph $G$ with vertex set $\sampleset$ and an edge between $x_i$
    and $x_j$ if $\norm{x_i - x_j} \leq r_c$.
  \item In decreasing order of size, query the label of each
    $\hat A_i$ until $\leq \frac{\epsilon}{4}n$ points belong to
    unlabeled clusters.
  \item Output $\hat f(x) =$ label of nearest labeled cluster to $x$.
  \end{enumerate}
\end{framed}

\caption{\vspace{0.25cm}Single-linkage learning.\vspace{-0.25cm}}
\label{alg:singlelinkage}
\end{myalg}

In order to get a meaningful reduction in label complexity, we need to
ensure that when we cluster a sample of data, most of the samples will
belong to a small number of clusters. For this purpose, we borrow the
following very general and interesting thick level set condition from
\citet{Steinwart2015}: a density function $p$ has $C$-thick level sets
if there exists a level $\lambda_0 > 0$ and a radius $\sigma_0 > 0$
such that for every level $\lambda \leq \lambda_0$ and radius
$\sigma < \sigma_0$, (1) the $\sigma$-interior of $\{p \geq \lambda\}$
is non-empty and (2) every point in $\{p \geq \lambda\}$ is at most
distance $C\sigma$ from the $\sigma$-interior. This condition
elegantly characterizes a large family of distributions for which
single-linkage style clustering algorithms succeed at recovering the
high-density clusters and only rules out distributions whose level
sets have bridges or cusps that are too thin. The thickness parameter
$C$ measures how pointed the boundary of the level sets of $p$ can
be. For example, in $\reals^d$ if the level set of $p$ is a ball then
$C=1$, while if the level set is a cube, then $C = \sqrt{d}$.

Using the thick level set condition to guarantee that our clustering
algorithm will not subdivide the high-density clusters of $p$, we
obtain the following result for Algorithm~\ref{alg:singlelinkage}

\begin{restatable}{thm}{HamDPOTheorem}
  \label{thm:hamDPO}
  Suppose that Assumption~\ref{assm:ecoc} holds and that the data
  distribution has $C$-thick level sets. For any target error
  $\epsilon > 0$, let $N$ be the number of connected components of
  $\{p \geq \epsilon/(2\vol(K))\}$. With probability at least
  $1-\delta$, running Algorithm~\ref{alg:singlelinkage} with parameter
  $r_c < g$ on an unlabeled sample of size
  $n = \tilde O(\frac{1}{\epsilon^2}((4C)^{2d}d^{d+1} / r_c^{2d} +
  N))$ will query at most $N$ labels and output a classifier with
  error at most $\epsilon$.
\end{restatable}

\begin{proof}
  For convenience, define $\sigma = r_c / (4C)$ and
  $\lambda = \epsilon / (2\vol(K))$. Using a standard VC-bound
  \citep{VC} together with the fact that balls have VC-dimension
  $d+1$, for $n = O((4C)^{2d} d^{d+1} / (\epsilon^2 r_c^{2d}))$
  guarantees that with probability at least $1-\delta/2$ the following
  holds simultaneously for every center $x \in \reals^d$ and radius
  $r \geq 0$:
  \begin{equation}
    \label{eq:ballconvergence}
    \biggl| |B(x,r) \cap \sampleset|/n - P(B(x,r)) \biggr|
    \leq \frac{1}{2} \lambda \sigma^d v_d,
  \end{equation}
  where $v_d$ denotes the volume of the unit ball in
  $\reals^d$. Assume that this high probability event occurs.

  We first show that the sample $\sampleset$ forms a
  $2C\sigma$-covering of the set $\{p \geq \lambda\}$; that is, for
  every $x \in \{p \geq \lambda\}$ we have
  $d(x,\sampleset) \leq 2C\sigma$. Let $x$ be any point in
  $\{p \geq \lambda\}$. Since $p$ has $C$-thick level sets, we know
  that there exists a point
  $y \in \interior{\{p \geq \lambda\}}{\sigma}$ such that
  $\norm{x - y} \leq C\sigma$. Moreover, the ball $B(y,\sigma)$ is
  contained in $\{p \geq \lambda\}$, which implies that it has
  probability mass at least $\lambda \sigma^d v_d$ and by
  \eqref{eq:ballconvergence} we have that
  $|B(y,r) \cap \sampleset|/n \geq \frac{1}{2} \lambda \sigma^d v_d >
  0$, so there must exist a point $z \in \sampleset \cap
  B(y,\sigma)$. Now we have that
  $d(x, \sampleset) \leq \norm{x - z} \leq \norm{x - y} + \norm{y - z}
  \leq C\sigma + \sigma \leq 2C\sigma$, where the final inequality
  follows from the fact that $C \geq 1$.

  Now let $A_1$, \dots, $A_N$ be the $N$ connected components of
  $\{p \geq \lambda\}$. We will argue that for each $i \in [N]$, there
  exists a unique cluster output by step 1 of the algorithm, say
  $\hat A_i$, such that $\hat A_i$ contains $A_i \cap \sampleset$ and
  for any point $x \in A_i$, the closest output cluster is $\hat A_i$.

  To see that $\hat A_i$ contains $A_i \cap \sampleset$, consider any
  pair of points $x$ and $x'$ in $A_i \cap \sampleset$. Since $A_i$ is
  connected, we know there is a path $\pi : [0,1] \to A_i$ such that
  $\pi(0) = x$ and $\pi(1) = x'$. Since the sample set $\X$ is a
  $2C\sigma$ covering of $\{p \geq \lambda\}$, it is also a
  $2C\sigma$-covering of $A_i$, which implies that we can find a
  sequence of points $y_1$, \dots, $y_M \in \X$ (possibly with
  repetition) such that the path $\pi$ passes through the balls
  $B(y_1, 2C\sigma)$, \dots, $B(y_M, 2C\sigma)$ in order. Since
  consecutive balls must touch at the point that the path $\pi$
  crosses from one ball to the next, we know that
  $\norm{y_i - y_{i+1}} \leq 4C\sigma = r_c$, and therefore the path
  $x \to y_1 \to \dots \to y_M \to x'$ is a path in the graph $G$
  connecting $x$ and $x'$.

  Now consider any point $x \in A_i$. We argued above that there
  exists a sample point $z \in \hat A_i$ that was within distance
  $2C\sigma$ from $x$. Now let $z^*$ be the closest sample in $\X$ to
  $x$. Then we know that
  $\norm{x - z^*} \leq \norm{x - z} \leq 2C\sigma$. By the triangle
  inequality, we have that
  $d\norm{z - z^*} \leq \norm{z - x} + \norm{x - z^*} \leq 4C\sigma
  \leq r_c$, and therefore $z$ and $z^*$ are connected in the graph
  $G$. Since $z$ belongs to $\hat A_i$, it follows that $z^*$ does
  too, and therefore the closest cluster to $x$ is $\hat A_i$.

  It remains to bound the error of the resulting classification
  rule. Since there is a margin of width $g > 0$ separating the
  classes, we know that every connected component of
  $\{p \geq \lambda\}$ must contain points belonging to exactly one
  class. Moreover, since we ran the algorithm with connection radius
  $r_c < g$, we know that the clusters output by step 1 will contain
  points belonging to exactly one class. It follows that if we query
  the label of any point in the cluster $\hat A_i$ then the algorithm
  will not error on any test point in $A_i$. Say that one of the
  connected components $A_i$ is labeled if we query the label of the
  corresponding cluster $\hat A_i$.

  Applying Hoeffding's inequality and the union bound to all possible
  $2^N$ unions of the sets $A_1$, \dots, $A_N$, our value of $n$
  guarantees that with probability at least $1-\delta/2$, the
  following holds simultaneously for all subsets of indices
  $I \subset [N]$:
  \[
    \biggl| \bigl| \sampleset \cap \bigl( \bigcup_{i \in I} A_i \bigr) \bigr|/n - P(\bigcup_{i \in I} A_i) \biggr| \leq \frac{\epsilon}{4}.
  \]
  Since the algorithm queries labels until at most
  $\frac{\epsilon}{4}n$ points belong to unlabeled clusters, we know
  that the number of samples belonging to the unlabeled $A_i$ sets is
  at most $\frac{\epsilon}{4}n$. By the above uniform convergence, it
  follows that their total probability mass is at most
  $\epsilon/2$. Finally, since the algorithm only errors on test
  points in $\{p \leq \lambda\}$, which has probability mass at most
  $\epsilon/2$ or on unlabeled $A_i$ sets, the error of the resulting
  classifier is at most $\epsilon$.
\end{proof}

The exponential dependence on the dimension in
Theorem~\ref{thm:hamDPO} is needed to ensure the sample $\sampleset$
will be a fine covering of the level set of $p$ w.h.p, which
guarantees that Algorithm~\ref{alg:singlelinkage} will not subdivide
its connected components into smaller clusters. When the data has low
intrinsic dimensionality, the unlabeled sample complexity is only
exponential in the intrinsic dimension. The following result shows
that under the common assumption that the distribution is a doubling
measure, then the unlabeled sample complexity is exponential only in
the doubling dimension. Recall that a probability measure $P$ is said
to have doubling dimension $D$ if for every point $x$ in the support
of $P$ and every radius $r > 0$, we have that
$P(B(x,2r)) \leq 2^D P(B(x,r))$ (see, for example,
\citep{Dasgupta2013}).

\begin{restatable}{thm}{HamDPODoublingTheorem}
  \label{thm:hamDPODoubling}
  Suppose that Assumption~\ref{assm:ecoc} holds the data distribution
  $P$ has doubling dimension $D$, and the support of $P$ has $N$
  connected components. With probability at least $1-\delta$, running
  Algorithm~\ref{alg:singlelinkage} with parameter $r_c < g$ on a
  sample of size $n = \tilde O\bigl( d/r_c^{2D} + N/\epsilon^2 \bigr)$
  will query at most $N$ labels and have error at most $\epsilon$.
\end{restatable}
\begin{proof}
  Let $x$ be any point in the support of $P$. Since we assumed that
  the diameter of $\X$ is 1, we know that $\X \subset B(x,1)$ and
  therefore $P(B(x,1)) = 1$. Applying the doubling condition $\lg(r)$
  times, it follows that for any radius $r > 0$ we have that
  $P(B(x,r)) \geq r^{-D}$.

  As in the proof of Theorem~\ref{thm:hamDPO}, for our choice of $n$
  the following holds with probability at least $1-\delta/2$ uniformly
  for every center $x$ in $\X$ and radius $r \geq 0$:
  \[
    \bigl| |B(x,r) \cap \sampleset|/n - P(B(x,r)) \bigr| \leq \frac{1}{2} r^{-D}.
  \]
  Assume this high probability event occurs. Since every ball of
  radius $r$ centered at a point in the support of $P$ has mass at
  least $r^{-D}$, each such ball must contain at least one sample
  point and it follows that the sample $\sampleset$ forms an
  $r^{-D}$-covering of the support of $P$.

  The rest of the proof now follows identically the proof of
  Theorem~\ref{thm:hamDPO} with the $A_1, \dots, A_N$ sets being the
  connected components of the support of $P$, since each connected
  component must be label-homogeneous.
\end{proof}

The unlabeled sample complexity in Theorem~\ref{thm:hamDPO} depends on
the gap $g$ between classes because we must have $r_c < g$. Such a
scale parameter must appear in our results, since
Assumption~\ref{assm:ecoc} is scale-invariant, yet our algorithm
exploits scale-dependent geometric properties of the problem. If we
have a conservatively small estimate $\hat g \leq g$, then the
conclusion of Theorem~\ref{thm:hamDPO} and
Theorem~\ref{thm:hamDPODoubling} continue to hold if the connection
radius and unlabeled sample complexity are set using the estimate
$\hat g$. Nevertheless, in some cases we may not have an estimate of
$g$, making it difficult to to apply
Algorithm~\ref{alg:singlelinkage}. The following result shows that if
we have an estimate of the number of high-density clusters, and these
clusters have roughly balanced probability mass, then we are still
able to take advantage of the geometric structure even when the
distance $g$ is unknown. The idea is to construct a hierarchical
clustering of $\sampleset$ using single linkage, and then to use a
small number of label queries to find a good pruning.

\begin{myalg}
\begin{framed}
  \hangindent=0.7cm {\bf Input:} Sample
  $\sampleset = \{x_1, \dots, x_n\}$, $t \in \mathbb{N}$.
  \begin{enumerate}[noitemsep,nolistsep,leftmargin=*]
  \item Let $T$ be the hierarchical clustering of $S$ obtained by
    single-linkage.
  \item Query the labels of a random subset of $\sampleset$ of size
    $t$.
  \item Let $\{\hat B_i\}_{i=1}^M$ be the coarsest pruning of $T$ such
    that each $\hat B_i$ contains labels from one class.
  \item Output $\hat f(x) = $ label of nearest $\hat B_i$ to $x$.
  \end{enumerate}
\end{framed}
\vspace{-0.25cm}
\caption{Hierarchical single-linkage learning.}
\vspace{-0.3cm}
\label{alg:activehierarchicallinkage}
\end{myalg}

\begin{restatable}{thm}{ActiveHierarchicalClusteringTheorem}
  \label{thm:active_hierarchical_clustering}
  Suppose Assumption~\ref{assm:ecoc} holds and the density $p$ has
  $C$-thick level sets. For any $0 < \epsilon \leq 1/2$, suppose that
  $\{A_i\}_{i=1}^N$ are the connected components of
  $\{p \geq \epsilon/(2\vol(K))\}$ and for some $\alpha \geq 1$ we
  have $P(A_i) \leq \alpha P(A_j)$ for all $i,j$. With probability
  $\geq 1-\delta$, running
  Algorithm~\ref{alg:activehierarchicallinkage} with
  $t = \tilde O(\alpha N)$ on an unlabeled sample of size
  $n = \tilde O (\frac{1}{\epsilon^2} (C^{2d}d^{d+1}/g^{2d} + N))$
  will have error $\leq \epsilon$.
\end{restatable}
\begin{proof}
  Define $\lambda = \epsilon / (2\vol(K))$ and let $A_1$, \dots, $A_N$
  be the connected components of $\{p \geq \lambda\}$. Suppose that
  each $A_i$ set has probability mass at least $\gamma$. Under the
  assumption that the probability mass of the largest $A_i$ is at most
  $\alpha$ times the mass of the smallest, we have that
  $\gamma \geq (1-\epsilon) / (\alpha N)$, but the result holds for
  any arbitrary lower bound $\gamma$.

  Since we query the labels of points without replacement, the set of
  labeled examples is an iid sample from the data density
  $p$. Whenever $m \geq \frac{2}{\gamma} \ln \frac{2N}{\delta}$, with
  probability at least $1-\delta/2$, every set $A_i$ will contain at
  least one labeled example, since they each have probability mass at
  least $\gamma$. Assume this high probability event holds.
  
  Let $g$ be the margin between classes that is guaranteed by
  Lemma~\ref{lem:separation}. Whenever samples $x, x' \in \sampleset$
  have $\norm{x - x'} \leq g$, they must belong to the same cluster
  $\hat B_i$. Applying an identical covering-style argument as in
  Theorem~\ref{thm:hamDPO}, we have that with probability at least
  $1-\delta/2$, for every $A_i$ set there is a cluster, say
  $\hat B_i$, such that:
  \begin{enumerate}
  \item All samples in $A_i \cap \sampleset$ are contained in $\hat
    B_i$.
  \item For every $x \in A_i$, the nearest cluster to $x$ is
    $\hat B_i$.
  \end{enumerate}

  Since every $A_i$ set contains at least one labeled example, it
  follows that whenever two of these high-density clusters belong to
  different classes, they will contain differently labeled points and
  therefore will not have been merged by
  Algorithm~\ref{alg:activehierarchicallinkage}. It follows that the
  label of $\hat B_i$ must agree with the label of $A_i$. At this
  point, the error analysis follows identically as in
  Theorem~\ref{thm:hamDPO}.
\end{proof}

In Section~\ref{sec:agnostic} we describe a meta-argument that can be
used to extend our results into the agnostic setting, where we no
longer require that the output code is consistent. Details for the
error correcting case are given in Section~\ref{apdx:hamdp1}.

In this section we showed that when there exists linear error
correcting correcting output code with low error, then it is possible
to reduce the label complexity of learning to the number of
high-density clusters, which are the connected components of
$\{p \geq \epsilon\}$. The label complexity of our algorithms is
always linear in the number of high density clusters, while the
worst-case unlabeled complexity of our algorithms is exponential in
the dimension (or intrinsic dimension).

%% file: one_vs_all.tex
\section{One-Versus-All on the Unit Ball}
\label{sec:onevsall}

\begin{figure}
  \centering
  \includegraphics[width=0.4\columnwidth]{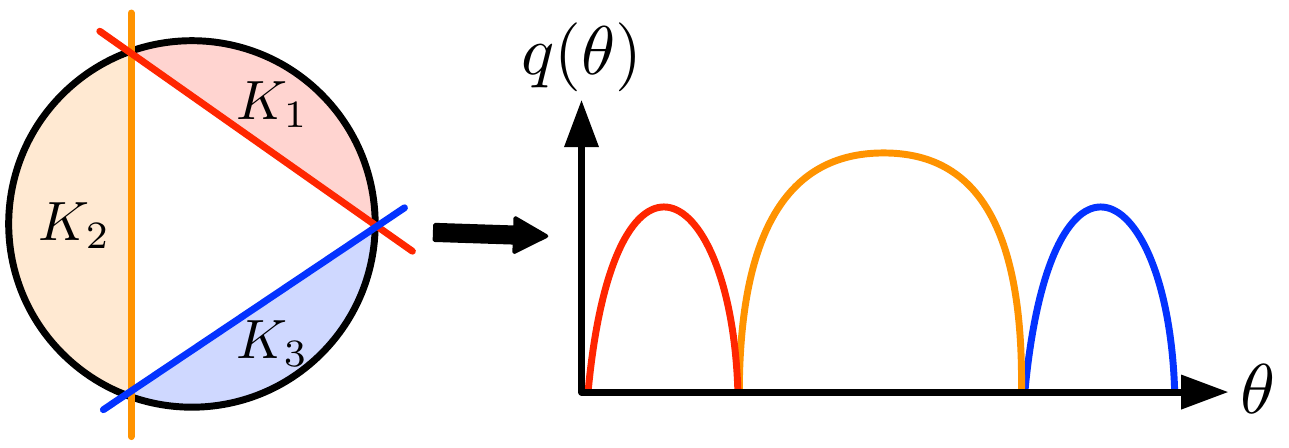}
  \caption{An example problem satisfying Assumption~\ref{assm:ova} and
    the projected density $q$ when the density $p$ is uniform on
    $K$.}
\label{fig:onevsall}
\end{figure}

In this section we show that even when the codewords are not well
separated, we can still exploit the implicit structure of output codes
to reduce the label complexity of learning by clustering the
data. Specifically, we consider the implicit structure of a linear
one-vs-all classifier over the unit ball:
\begin{assm}
  \label{assm:ova}
  The instance space $\X$ there exist $L$ linear separators $h_1$,
  \dots, $h_L$ such that: (1) point $x$ belongs to class $i$ iff
  $h_i(x) > 0$, and (2) for all $i$, $h_i(x) = w_i^\tp x - b_i$ with
  $\norm{w_i} = 1$ and $b_i \geq \bmin > 0$.
\end{assm}

See Figure~\ref{fig:onevsall} for an example problem satisfying this
condition. Since a one-vs-all classifier is an output code where the
code matrix is the identity, the Hamming distance between any pair of
codewords is exactly 2. Therefore, in this setting we do not have a
result similar to Lemma~\ref{lem:separation} to ensure that the
classes are geometrically separated. Instead, we exploit the
one-vs-all structure to show the classes are probabilistically
separated and employ a robust clustering algorithm.

As before, we study this problem under a mild constraint on the data
distribution. For each class $i$ denote the set of points in class $i$
by $K_i = \setc{x}{\norm{x} \leq 1, h_i(x) > 0}$ and
$K = \bigcup_{i=1}^L K_i$. In this section, we assume that the density
$p$ is supported on $K$ with upper and lower bounds:
\begin{assm}
  \label{assm:nearlyuniform}
  There exist constants $0 < \clb \leq \cub$ s.t. for $x \in K$ we
  have $\clb \leq p(x) \leq \cub$ and otherwise $p(x) = 0$.
\end{assm}
This distributional constraint is quite general: it only requires that
we will not observe examples for which the one-vs-all classifier would
be confused (i.e., where none of its linear separators claim the
point) and that the density does not take extreme values. When $K$ is
compact, every continuous density supported on $K$ satisfies
Assumption~\ref{assm:nearlyuniform}.
  
Our algorithm for this setting first projects the data onto the unit
sphere $\unitsphere = \setc{x \in \reals^d}{\norm{x} = 1}$ and then
applies a robust clustering algorithm to the projected data. The
projection does not introduce any errors, since the label of an
example is independent of its distance to the origin. This is because
each linear separator carves out a spherical cap for its class, and no
two class caps overlap. Since we assume that no class contains the
origin, it follows that an example’s label depends only on its
projection to the sphere. We show that projecting to the sphere has
the useful property that the projected density goes to zero at the
boundary of the classes, which suggests that we can use a robust
single-linkage style clustering algorithm to find label-homoegeneous
clusters. Algorithm~\ref{alg:ova} gives pseudocode, using the notation
$\ang(u,v) = \arccos(u^\tp v)$ for the angle between $u$ and $v$ and
$V^d(r)$ is the probability that a uniformly random sample from
$\unitsphere$ lands in a given spherical cap of angular radius $r$.

\begin{myalg}[htb]
  \begin{framed}
    \hangindent=0.7cm {\bf Input:} Sample
    $\sampleset = \{x_1, \dots, x_n\}$, radius $r_c >0$.
    \begin{enumerate}[noitemsep,nolistsep,leftmargin=*]
    \item Define $r_a = r_c/2$ and $\tau = \frac{\clb}{2\cub}V^d(r_a)\epsilon$.
    \item Let $v_i = \frac{x_i}{\norm{x_i}}$ be the projection of
      $x_i$ to the sphere.
    \item Mark $v_i$ active if
      $|\setc{v_j}{\ang(v_i,v_j) \leq r_a}| \geq \tau n$ and inactive
      otherwise for $i \in [n]$.
    \item Let $\hat A_1$, \dots, $\hat A_N$ be the connected
      components of the graph $G$ whose vertices are the active $v_i$
      with an edge between $v_i$ and $v_j$ if $\ang(v_i, v_j) < r_c$.
    \item In decreasing order of size, query the label of each
      $\hat A_i$ until $\leq \frac{\epsilon}{4}n$ points belong to
      unlabeled clusters.
    \item Output $\hat f(x) =$ label of nearest cluster to
      $x/\norm{x}$.
    \end{enumerate}
  \end{framed}
  \caption{Robust single-linkage learning.\vspace{-0.2cm}}
  \label{alg:ova}
\end{myalg}

Our first result characterizes the density of the projected data
(defined relative to the uniform distribution on $\unitsphere$).

\begin{restatable}{lem}{ProjectedDensityLemma}
\label{lem:projectedDensity}
Suppose Assumptions~\ref{assm:ova} and \ref{assm:nearlyuniform} hold
and let $q:\unitsphere \to [0,\infty)$ be the density function of the
data projected onto the unit sphere. Then
$\ql(v) \leq q(v) \leq \qu(v)$, where
\[
  \ql(v) = \begin{cases}
    \clb d v_d (1 - (b_i / w_i^\tp v)^d) & \hbox{if $v \in K_i$} \\
    0 & \hbox{otherwise},
  \end{cases}
\]
and $\qu(v) = \cub/\clb \cdot \ql(v)$, where $v_d$ is the volume of
the unit ball in $d$ dimensions.
\end{restatable}

\begin{proof}
  Let $X \sim p$ be and set $V = X / \norm{X}_2$ so that $V$ is a
  sample from $q$.  For any set $A \subset \sphere$, we know that
  $\prob(V \in A) = \prob(X \in \cone(A))$, where
  $\cone(A) = \setc{rv}{r > 0, v \in A}$, which gives
  \[
    \prob(V \in A)
    = \prob(X \in \cone(A)) 
    = \int_{x \in \cone(A)} p(x)\, dx 
    = \int_{v \in A}  d v_d \int_{r = 0}^\infty p(rv)r^{d-1} \, dr \, d\sam(v),
  \]
  where the last inequality follows by a change of variables $x$ to
  $(r,v)$ where $r = \norm{x}_2$ and $v = x / \norm{x}_2$. The term
  $r^{d-1}$ is the determinant of the Jacobian of the change of
  variables, and the term $d v_d$, which is the surface area of
  $\sphere$, appears since $\sam$ is normalized so that
  $\sam(\sphere) = 1$. From this, it follows that the density function
  $q$ can be written as
  \begin{equation}
    \label{eq:qintegral}
    q(v) = d v_d \int_{r = 0}^\infty p(rv)r^{d-1} \, dr,
  \end{equation}
  since integrating this function over any set $A$ gives the
  probability that $V$ will land in $A$. From our assumptions on $p$,
  we know that
  \[
    p(rv) \geq \sum_{i=1}^L \ind{rv \in K_i} \clb.
  \]
  Moreover, we can rewrite the indicator as $\ind{rv \in K_i} =
  \ind{\frac{b_i}{w_i^\tp v} < r \leq 1}$. Substituting this into
  \eqref{eq:qintegral} gives
  \begin{align*}
    q(v)
    &\geq \sum_{i=1}^L \clb d v_d \int_{r=0}^\infty \ind[\bigg]{\frac{b_i}{w_i^\tp v} < r \leq 1} r^{d-1}\, dr \\
    &= \sum_{i=1}^L \ind{v \in K_i} \clb d v_d \int_{r=b_i/(w_i^\tp v)}^1  r^{d-1}\, dr \\
    &= \sum_{i=1}^L \ind{v \in K_i} \clb d v_d (1 - b_i / (w_i^\tp v)^d) \\
    &= \ql(v)
  \end{align*}
  Note that the indicator $\ind{v \in K_i}$ appears in line 2 because
  the integral is only non-zero when $b_i / (w_i^\tp v) < 1$, which is
  exactly the condition that $v \in K_i$. The upper bound on $q$
  follows by an identical argument using the upper bound on $p(rv)$.
\end{proof}

Both bounds are defined piece-wise with one piece for each
class. Restricted to class $i$, both the $\ql(v)$ and $\qu(v)$ are
decreasing functions of $\ang(w_i, v)$, which implies that their
$\lambda$-level sets are spherical caps. Therefore, each class
contributes one large connected component to the level set of $q$ that
is roughly a spherical cap centered at the point $w_i$ and the density
of $q$ goes to zero at the boundary of each class. Our main result is
as follows:

\begin{restatable}{thm}{onevsallTheorem}
  \label{thm:onevsall}
  Suppose Assumptions~\ref{assm:ova} and \ref{assm:nearlyuniform} hold
  and that $f^*$ is consistent. There exists an $r_c$ satisfying
  $r_c = \Omega( \epsilon \clb / (\cub^2\bmin))$ such that with
  probability at least $1-\delta$, running Algorithm~\ref{alg:ova}
  with parameter $r_c$ on an unlabeled sample of size
  $n = \tilde O((\cub^4d/(\epsilon^2 \clb^2 \bmin^2))^d)$ will query
  at most $L$ labels and output a classifier with error at most
  $\epsilon$.
\end{restatable}

Note that if the scale parameter $\bmin$ is unknown, the conclusion of
Theorem~\ref{thm:onevsall} continues to hold if the connection radius
$r_c$ and unlabeled sample complexity $n$ are set using a
conservatively small estimate $\widehat \bmin$ satisfying
$\widehat \bmin \leq \bmin$. This comes at the cost of an increased
unlabeled sample complexity.

Before proving Theorem~\ref{thm:onevsall}, we develop some general
results for the robust linkage clustering algorithm.  More generally,
Algorithm~\ref{alg:ova} can be applied in any metric space
$(\X, \dist)$ by replacing $\ang$ with the distance metric $d$ and
suitable settings for the internal parameters $r_a$ and $\tau$. For
the robust linkage approach to have low error, each class should have
one large connected component in the graph $G$ constructed by the
algorithm so that: (1) with high probability a new point in class $i$
will be nearest to that largest component, and (2) the large
components of different classes are separated. Intuitively, $G$ will
have these properties if each positive region $\K_i$ has a connected
high-density inner region $A_i$ covering most of its probability mass
and when it is rare to observe a point that is close to two or more
classes. This notion is formalized below.

Let $S$ be any set in $\X$. We say that a path $\pi : [0,1] \to \X$
\emph{crosses} $S$ if the path starts and ends in different connected
components of the complement of $S$ in $\X$ and we say that the
\emph{width} of $S$ is the length of the shortest path that crosses
$S$.

\begin{defn}
  The sets $A_1$, \dots, $A_\numclasses$ are $(r_c, r_a, \tau,
  \gamma)$-clusterable under probability $P$ if there exists a separating set
  $S$ of width at least $r_c$ such that:
  (1) Each $A_i$ is connected;
  (2) If $x \in \X$ satisfies $\dist(x, A_i) \leq r_c / 3$ then $\prob_{X \sim P}(X \in \ball(x, r_a)) > \tau + \gamma$;
  (3) If $x \in A_i$ then $\prob_{X \sim P}(X \in \ball(x, r_c/3)) > \gamma$;
  (4) Every path from $A_i$ to $A_j$ crosses $S$; and
  (5) If $x \in S$ then $\prob_{X \sim P}(X \in \ball(x,r_a)) < \tau - \gamma$.
\end{defn}

Note that typically there must be a gap between the set $A_i$ and the
set $S$ in order to satisfy the probability requirements (i.e., the
set $S$ will be smaller than $\X - \bigcup_{i=1}^L A_i$). The first
three properties ensure that each set $A_i$ will have one large
connected component and the remaining two properties ensure that these
connected components will be disconnected. Following an analysis
similar to that of the cluster tree algorithm of \citet{Chaudhuri2010}
gives the following result.

\begin{restatable}{lem}{RobustLinkageLemma}
  \label{lem:robustlinkage}
  Suppose that the sets $A_1$, \dots, $A_N$ are
  $(r_c, r_a, \tau, \gamma)$-clusterable with respect to distribution
  $P$. For any failure probability $\delta > 0$, let $G$ be the graph
  constructed by Algorithm~\ref{alg:ova} run on a sample $\sampleset$
  of size $O(\frac{1}{\gamma^2}(D + \ln \frac{1}{\delta})$, where $D$
  is the VC-dimension of balls in $(\X, d)$, with parameters and
  $r_c$, $r_a$, and $\tau$. Define
  $K_i = \setc{x \in S}{d(x,A_i) \leq r_c/3}$ for each $i \in
  [N]$. With probability at least $1-\delta$, the graph $G$ has the
  following properties:
  \begin{enumerate}
  \item {\bf Complete:} For each $i$, all samples in $K_i$ are active and
    included in the graph $G$.
  \item {\bf Separated:} For any $i \neq j$, there is no path in $G$ from
    $K_i$ to $K_j$.
  \item {\bf Connected:} For every $i$, the set $K_i$ is connected in $G$.
  \item {\bf Extendible:} For any point $x \in A_i$, the nearest
    connected component of $G$ to $x$ contains $K_i$.
  \end{enumerate}

\end{restatable}
\begin{proof}
  The proof technique used here follows a similar argument as
  \citet{Chaudhuri2010}.

  We use a standard VC bound \citep{VC} to relate the probability
  constraints in the clusterability definition to the empirical
  measure $\hat P$. For our value of $n$ we have
  \[
  \prob\bigl(
    \sup_{x,r} \bigl| \hat P(\ball(x,r)) - P(\ball(x,r)) \big| > \gamma
  \bigr) < \delta.
  \]
  This implies that with probability at least $1-\delta$ for all
  points $x$ we have: (1) if $\dist(x,A_i) \leq \frac{r_c}{3}$ for any
  $i$ then $\hat P(B(x,r_a)) > \tau$; (2) if $x \in S$ then
  $\hat P(B(x,r_a)) < \tau$; and (3) if $x \in A_i$ for any $i$ then
  $\hat P(B(x,\frac{r_c}{3})) > 0$.  We now use these facts to prove
  that the graph $G$ has the completeness, separation, and
  connectedness properties.

  Completeness follows from the fact that every sample
  $x \in \hat K_i$ is within distance $r_c/3$ of $A_i$ and therefore
  $\hat P(B(x, r_a)) > \tau$.

  To show separation, first observe that every sample
  $z \in \sampleset$ that belongs to $S$ will be marked as inactive,
  since $\hat P(\ball(z, r_a)) < \tau$. Now let $x \in \hat K_i$ and
  $x' \in \hat K_j$ for $i \neq j$. Since the graph $G$ does not
  contain any samples in the set $S$, any path in $G$ from $x$ to $x'$
  must have one edge that crosses $S$. Since the width of $S$ is at
  least $r_c$, this edge would not be included in the graph $G$, and
  therefore $G$ does not include a path from $x$ to $x'$.

  To show connectedness, let $x$ and $x'$ be any pair of samples in
  $\hat K_i$ and let $v$ and $v'$ be their nearest points in $A_i$,
  respectively. By definition of $\hat K_i$, we know that
  $d(x, v) < r_c/3$ and $d(x', v') < r_c/3$. Since $A_i$ is a
  connected set, there is a path $\pi : [0,1] \to A_i$ in $A_i$
  starting at $v$ and ending at $v'$. Cover the path $\pi$ with a
  sequence of points $z_1$,~\dots,~$z_k$ such that
  $d(z_j, z_{j+1}) < r_c/3$ for all $j$ and the path $\pi$ is covered
  by the balls $\ball(z_j, r_c/3)$.  Further, choose $z_1 = v$ and
  $z_k = v'$. Since each point $z_j$ belongs to $A_i$, the empirical
  probability mass of the ball $\ball(z_j, r_c/3)$ is non-zero, which
  implies that it must contain at least one sample point, say
  $y_j \in \sampleset$. We may take $y_1 = x$ and $y_k = x'$. Since
  every sample $y_1$, \dots, $y_k$ is within distance $r_c/3$ of
  $A_i$, they are all active and included in the graph $G$. Moreover,
  since $d(y_j, y_{j+1}) < r_c$, we have that the path
  $x = y_1 \to \dots \to y_k = x'$ is a path connecting $x$ and $x'$
  in $G$, as required.
 
  Finally to show extensibility, let $x \in A_i$ be any point. By the
  uniform convergence for balls, $P(x, r_c/3)$ has non-zero empirical
  probability mass and therefore contains at least one active sample,
  say $z$. Since $z$ is within distance $r_c/3$ of $A_i$, it belongs
  to the set $K_i$. Now let $z^*$ be the closest active sample to
  $x$. We must have $d(x, z^*) \leq d(x, z) \leq r_c/3$ and it follows
  that $d(z, z^*) \leq d(z, x) + d(x, z^*) \leq 2r_c / 3 <
  r_c$. Therefore, $z^*$ also belongs to $K_i$, as required.
\end{proof}

We now prove Theorem~\ref{thm:onevsall} by combining
Lemmas~\ref{lem:projectedDensity} and \ref{lem:robustlinkage}:
\begin{proof}[Proof of Theorem~\ref{thm:onevsall}]
  For each class $i \in [L]$, define
  $A_i = \{\qub{i} \geq \epsilon\}$. We will show that the sets $A_1$,
  \dots, $A_L$ are $(r_c, r_a, \gamma, \tau)$-clusterable for
  appropriate choices of the parameters. Then
  Lemma~\ref{lem:robustlinkage} will guarantee that with high
  probability, the clustering produced by Algorithm~\ref{alg:ova} will
  approximate the connected components of the $\epsilon$-level of
  $\{\qu \geq \epsilon\}$.

  Recall that for each class $i \in [L]$, the sets
  $\{\qub{i} \geq \epsilon\}$ and $\{\qlb{i} \geq \epsilon\}$ are
  spherical caps. To simplify notation, let
  $\scap(u,r) = \setc{v \in \unitsphere}{\ang(v,u) \leq r}$ denote the
  spherical cap of angular radius $r$ centered at $u$. Let
  $\rub{i}(\lambda) = \arccos(b_i(1 - \lambda / (\cub d v_d))^{-1/d})$
  denote the angular radius of $\{\qub{i} \geq \epsilon\}$, so that
  $\{\qub{i} \geq \epsilon\} = \scap(w_i, \rub{i}(\epsilon))$, and
  $\rlb{i}(\lambda)$, defined similarly, be the angular radius of
  $\{\qlb{i} \geq \epsilon\}$. Define
  $\tilde \epsilon = \frac{\clb}{\cub}\epsilon$ and suppose for the
  moment that we can find an activation radius $r_a > 0$ small enough
  so that the following inequalities hold for all classes
  $i = 1, \dots, L$:
  \[
    \frac{5}{3} r_a \leq \rlb{i}\bigr(\frac{3\tilde \epsilon}{4}\bigl)
    - \rub{i}(\epsilon) \qquad \hbox{and} \qquad 2 r_a \leq \rub{i}(0)
    - \rub{i}\bigl(\frac{\tilde \epsilon}{4}\bigr).
  \]
  Given such an activation radius, we will show that the sets $A_1$,
  \dots, $A_L$ are $(r_c, r_a, \tau, \gamma)$-clusterable with
  $r_c = 2r_a$, $\tau = \frac{\tilde \epsilon \scapvol(r_a)}{2}$, and
  $\gamma = \frac{\tilde \epsilon \scapvol(r_c/3)}{4}$ and the
  separating set is
  $S = \setc{v \in \sphere}{\ang(v, w_i) \geq \rub{i}(0) - r_a \hbox{
      for all $i$}}$:
\begin{enumerate}

\item {\it Connection:} Each $A_i$ set is a spherical cap and therefore connected.

\item {\it High-density near $A_i$:} Let $v \in \sphere$ be such that $\ang(v, A_i) <
r_c / 3$ and let $u \in \scap(v, r_a)$ be any point in the spherical cap of
angular radius $r_a$ about $v$. By the triangle inequality, we know that
$\ang(w_i, u) \leq \ang(w_i, v) + \ang(v, u) \leq \rub{i}(\epsilon) +
\frac{5}{3} r_a \leq \rlb{i}(\frac{3\tilde \epsilon}{4})$. This implies that
$q(u) \geq \frac{3 \tilde \epsilon}{4}$ for all points in $\scap(v,r_a)$ and
therefore $\prob_{V \sim q}(V \in \scap(v, r_a)) \geq \frac{4 \tilde
\epsilon}{3} \scapvol(r_a) \geq \tau + \gamma$.

\item {\it High-density inside $A_i$:} Now let $v \in A_i$. Since
  $r_c / 3 < r_a$, the above arguments show that
  $q(u) \geq \frac{4 \tilde \epsilon}{3}$ for all points
  $u \in \scap(v, r_c / 3)$ and therefore
  $\prob_{V \sim q}(V \in \scap(v, r_c/3)) \geq \frac{4\tilde
    \epsilon}{3} \scapvol(r_c/3) \geq \gamma$.

\item {\it Separation by the set $S$:} For each class $i$, the set $S$
  contains the annulus
  $\setc{v \in \sphere}{\rub{i}(0) - r_a \leq \ang(w_i, v) \leq
    \rub{i}(0)}$ which has width $r_a$. Any path from one $A_i$ to
  another $A_j$ must cross two such annuli, each of width $r_a$, so
  the length of the path crossing $S$ is at least $2r_a = r_c$.

\item {\it Low density inside $S$:} Finally, let $v$ be any point in
  the set $S$ and let $u \in \scap(v, r_a)$. For any class $i$, the
  reverse triangle inequality gives that
  $\ang(w_i, v) \geq \ang(v, w_i) - \ang(u, w_i) \geq \rub{i}(0) -
  2r_a \geq \rub{i}(\frac{\tilde \epsilon}{4})$. Since this is true
  for all classes $i$, we have $q(v) \leq \frac{\tilde \epsilon}{4}$
  and therefore
  $\prob_{V \sim q}(V \in \scap(v, r_a)) \leq \frac{\tilde
    \epsilon}{4} \scapvol(r_a) \leq \tau - \gamma$.

\end{enumerate}

It follows that the sets $A_1$, \dots, $A_L$ are
$(r_c, r_a, \tau, \gamma)$-clusterable and it only remains to find an
activation radius $r_a$ that satisfies the above inequalities. Since
the robust linkage algorithm needs to estimate the probability mass of
balls to within error
$\gamma = \frac{\tilde \epsilon \scapvol(2r_a/3)}{4}$, we want this
activation radius to be not too small.

Taking the first order Taylor expansion of the $\rlb{i}$ and $\rub{i}$
functions, we have:
\begin{align*}
\rlb{i}(\lambda) &= \arccos(b_i) - \frac{b_i}{\sqrt{1 - b_i^2}} \frac{1}{\clb d v_d} \lambda + O(\lambda^2) \\
\rub{i}(\lambda) &= \arccos(b_i) - \frac{b_i}{\sqrt{1 - b_i^2}} \frac{1}{\cub d v_d} \lambda + O(\lambda^2),
\end{align*}
as $\lambda \to 0$. Therefore, we have that
\[
\rlb{i}(3\tilde \epsilon/4) - \rub{i}(\epsilon) =
\frac{1}{4dv_d \cub}\cdot\frac{b_i}{\sqrt{1-b_i^2}}\epsilon + O(\epsilon^2)
\]
and
\[
\rub{i}(0) - \rub{i}(\tilde \epsilon / 4) =
\frac{\clb}{4dv_d\cub^2} \cdot \frac{b_i}{\sqrt{1-b_i^2}} \epsilon + O(\epsilon^2),
\]
which shows that it is sufficient to set
$r_a = \frac{3\clb}{20 dv_d\cub^2} \cdot \frac{b_i}{\sqrt{1-b_i^2}}
\epsilon + O(\epsilon^2) = \Omega(\frac{\clb}{\cub^2} \bmin \epsilon)$
as $\epsilon \to 0$ and it follows that
$n = O(\frac{1}{\gamma^2}(d + \ln \frac{1}{\delta}) = \tilde
O((\cub^4d/(\epsilon^2 \clb^2 \bmin^2))^d)$.

Finally, we show that the algorithm correctly recovers the labels of
the large clusters. For $n = \tilde O( L / \epsilon^2 )$, we have that
with probability at least $1-\delta$ the following holds
simultaneously for all $2^L$ subsets $I \subset [L]$:
$\bigl| \hat P(A_I) - P(A_I) \bigr| \leq \epsilon / 4$, where
$A_I = \bigcup_{i \in I} A_i$. Since all samples in $A_I$ are marked
as active (by Lemma~\ref{lem:robustlinkage}), this implies that all
but at most $\frac{\epsilon}{4}n$ of the active points will belong to
the $A_i$ sets. It follows that if the algorithm queries the labels of
the $L$ largest clusters, they will also contain all but
$\frac{\epsilon}{4}n$ active samples.

On the other hand, whenever we query the label of one of the $A_i$
sets, we know that we will correctly classify every test point
belonging to $A_i$, so the error of the resulting classifier is at
most the probability mass of $\{\qu \leq \epsilon\}$ together with the
probability mass of the $A_i$ sets for which we did not query the
label. Since the unqueried $A_i$ sets have empirical probability mass
at most $\epsilon / 4$ and we have uniform convergence for all unions
of $A_i$ sets to within error $\epsilon / 4$, it follows that the
total probability mass of the unlabeled $A_i$ sets is at most
$\epsilon / 2$ and it follows that the error of the resulting
classifier is at most $\epsilon$.
\end{proof}

There are two main differences between the sample complexity of
Theorem~\ref{thm:onevsall} and the results from
Section~\ref{sec:hamdp1}. First, the unlabeled sample complexity now
has an $\epsilon^{-2d}$ dependence, rather than only
$\epsilon^{-2}$. This is because the distance between the connected
components of $\{p \geq \epsilon\}$ goes to zero (in the worst case)
as $\epsilon \to 0$, so our algorithm must be able to detect
low-density regions of small width. In contrast,
Lemma~\ref{lem:separation} allowed us to establish a non-diminishing
gap $g > 0$ between the classes when the codewords were well
separated. On the other hand, the label complexity in this setting is
better, scaling with $L$ instead of $N$, since we are able to
establish that each class will have one very large cluster containing
nearly all of its data.

Theorem~\ref{thm:onevsallAgnostic} in the appendix gives an analysis
of Algorithm~\ref{alg:ova} in the agnostic setting of
Section~\ref{sec:agnostic}.

%% file: boundary_features.tex
\section{The Boundary Features Condition}
\label{sec:boundary_features}

\begin{figure}[b]
  \centering
  \includegraphics[width=0.4\columnwidth]{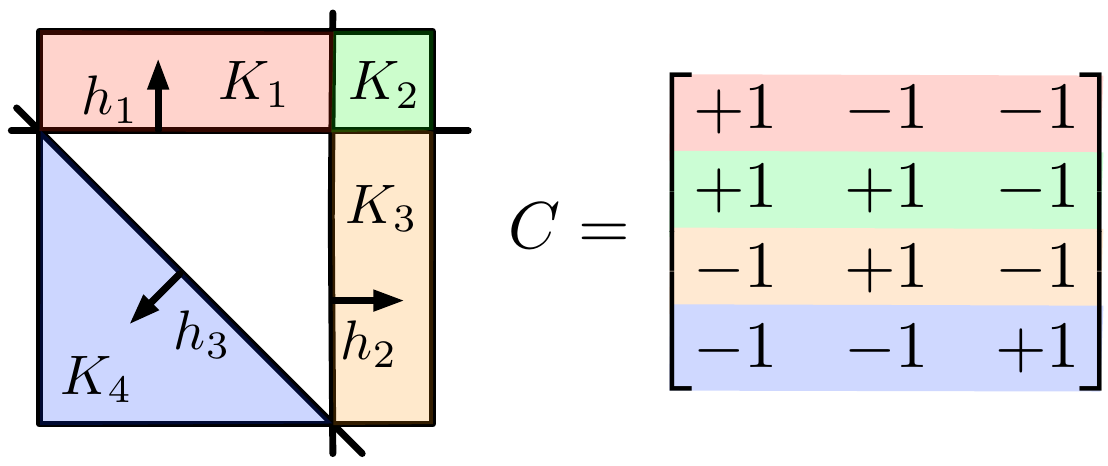}
  \caption{An example of the boundary features problem. The arrows
    indicate the positive side of the linear functions. \vspace{-0.4cm}}
  \label{fig:boundaryfeatures}
\end{figure}

Finally, in this section we introduce a novel condition on the code
matrix called the boundary features condition that captures the
intuition that every binary classification task should be
significant. Assumption~\ref{assm:boundaryfeatures} formalizes this
intuition.

\begin{assm}
  \label{assm:boundaryfeatures}
  There exists a code matrix $C \in \{\pm 1\}^{L \times m}$, linear
  functions $h_1$, \dots, $h_m$, and a scale parameter $R > 0$ so
  that: (1) for any point $x$ in class $y$, we have $h(x) = C_y$; (2)
  for each $h_j$, there exists a class $i$ such that negating the
  $j^{\rm th}$ entry of $C_i$ produces a codeword $C_i'$ not in $C$
  and there exists a point $x$ on the hyperplane $h_j = 0$ such that
  every point in $B(x,R)$ has either code word $C_i$ or $C_i'$; and
  (3) any pair of points $x, x' \in \X$ such that $h(x)$ and $h(x')$
  are not codewords in $C$ and $h(x) \neq h(x')$ must have
  $\norm{x - x'} \geq R$.
\end{assm}

Part~(1) of this assumption requires that the output code classifier
is consistent, part~(2) is a condition that guarantees every linear
separator $h_j$ separates at least one class $i$ from a region of
space that does not belong to any class, and part~(3) requires that
points with codewords not in the code matrix must either have the same
codeword or be separated by distance $R$. Part~(3) allows us to
simplify our algorithm and analysis and is trivially satisfied in
cases where all points in $\X$ that do not belong to any class have
the same codeword, as is the case for one-vs-all classification and
the problem in Figure~\ref{fig:boundaryfeatures}.

Problems in this setting are more challenging than those of the
previous sections because they may not be amenable to clustering-based
learning strategies. Whenever the Hamming distance between a pair of
codewords is only 1, this implies that one of the linear separators
$h_j$ forms a shared boundary between the classes, and therefore these
classes may be connected by a large and high-density region.  Instead,
Assumption~\ref{assm:boundaryfeatures} guarantees that for every
linear separator $h_j$, there is some ball $B(x,R)$ centered on $h_j$
that is half-contained in the set of points belonging to some class
$i$ and the other half belongs to the set of points that do not belong
to any class. Therefore, by looking for hyperplanes that locally
separate sample data from empty regions of space, we can recover the
linear separator $h_j$ from the local \emph{absence} of data. Define
$K_i = \setc{x \in \X}{h(x) = C_i}$ to be the set of points that
belong to class $i$ and $K = \bigcup_{i=1}^L K_i$. Under the condition
that the density $p$ is supported on $K$ and is upper and lower
bounded, we exploit this structure in an algorithm that directly
learns the linear separators $h_1$, \dots, $h_m$.

Our hyperplane detection algorithm works by searching for balls of
radius $r$ whose centers are sample points such that one half of the
ball contains very few samples. If a half-ball contains very few
sample points then it must be mostly disjoint from the set $\K$. But
since its center point belongs to the set $\K$, this means that the
hyperplane defining the half-ball is a good approximation to at least
one of the true hyperplanes. See
Figure~\ref{fig:qualifying_half_balls} for examples of half-balls that
would pass and fail this test. The collection $H$ of hyperplanes
produced in this way partition the space into cells. Our algorithm
queries the labels of the cells containing the most sample points and
classifies test points based on the label of their cell in the
partition (and if the label is unknown, we output a random
label). Pseudocode is given in Algorithm~\ref{alg:edgedetection} using
the following notation: for any center $x \in \X$, radius $r \geq 0$,
and direction $w \in \sphere$, let
$\halfball(x, r, w) = \setc{y \in \ball(x,r)}{w^\tp(y - x) > 0}$ and
define $\hbp(r) = \frac{1}{2} \clb r^d v_d$.

\begin{figure}
  \begin{minipage}{0.5\textwidth}
    \begin{framed}
      \hangindent=0.7cm {\bf Input:} Sample
      $\sampleset = \{x_1, \dots, x_n\}$, $r > 0$, $\tau > 0$.
      \begin{enumerate}[noitemsep,nolistsep,leftmargin=*]
      \item Initialize set of candidate hyperplanes $H = \emptyset$.
      \item For all samples $\hat x \in \sampleset$ with $\ball(\hat x, r) \subset \X$:
        \begin{enumerate}[noitemsep,nolistsep,leftmargin=*]
        \item Let $\hat w = \argmin_{w \in \sphere} |\halfball(\hat x, r, w) \cap \sampleset|$.
        \item If
          $|\halfball(\hat x, r, \hat w) \cap \sampleset| / n < \tau$, add
          $(\hat x, \hat w)$ to $H$.
        \end{enumerate}
      \item Let $\{\hat C_i\}_{i=1}^N$ be the partitioning of $\X$ induced
        by $H$.
      \item Query the label of the $L$ cells with the most samples.
      \item Output $\hat f(x)$ = label of $C_i$ containing $x$.
      \end{enumerate}
    \end{framed}
    
    \renewcommand{\figurename}{Algorithm}
    \setcounter{figure}{3}
    \captionof{figure}{Plane-detection algorithm.}
    \label{alg:edgedetection}
    \setcounter{figure}{2}
    \renewcommand{\figurename}{Figure}
  \end{minipage}
  \begin{minipage}{0.5\textwidth}
    \centering
    \vspace{2em}
    \includegraphics[width=0.55\columnwidth]{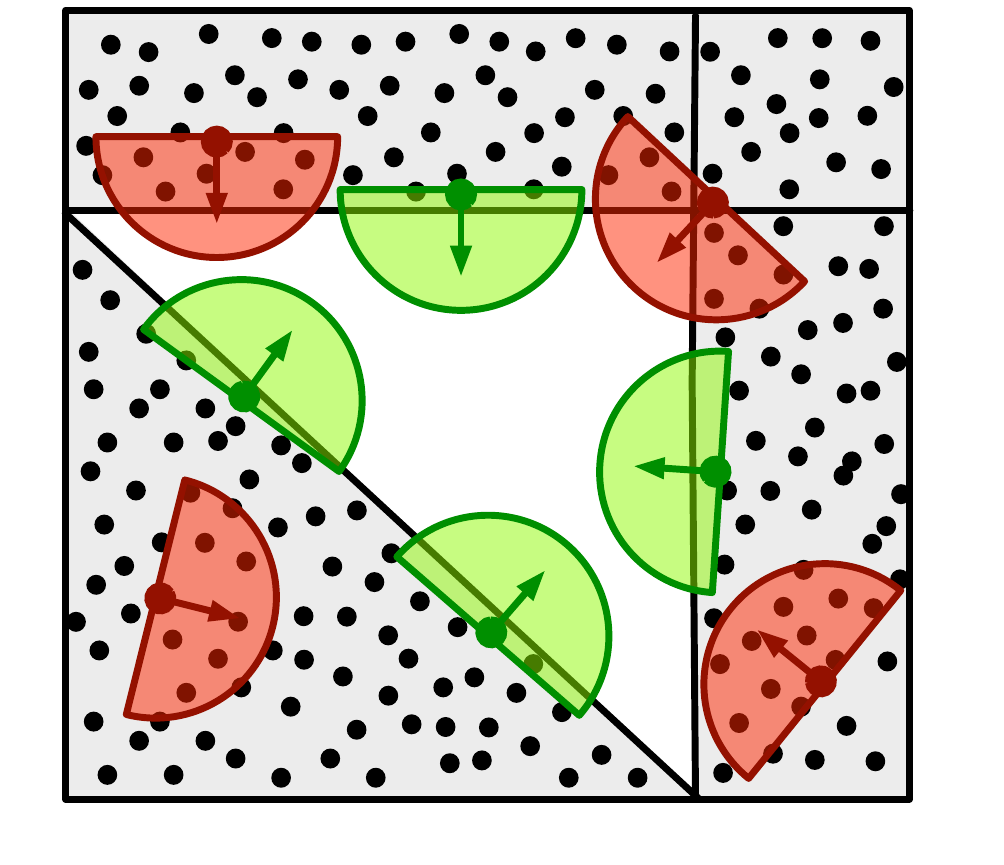}
    \captionof{figure}{Examples of half-balls that would be included (green) or excluded
      (red) by the plane detection algorithm.}
    \label{fig:qualifying_half_balls}
  \end{minipage}
\end{figure}

Each candidate hyperplane produced by
Algorithm~\ref{alg:edgedetection} is associated with a half-ball that
caused it to be included in $H$. In fact, we can think of the pairs
$(\hat x, \hat w)$ in $H$ as either encoding the linear function
$\hat h(x) = w^\tp (x - \hat x)$ or the half-ball
$\halfball(\hat x, r, \hat w)$, where $r$ is the scale parameter of
the algorithm. Most of our arguments will deal with the half-balls
directly, so we adopt the second interpretation. The analysis of
Algorithm~\ref{alg:edgedetection} has two main steps. First, we show
that the face of every half-ball in the set $H$ is a good
approximation to at least one of the true hyperplanes, and that every
true hyperplane is well approximated by the face of at least one
half-ball in $H$.  Second, using the fact that the half-balls in $H$
are good approximations to the true hyperplanes, we argue that the
output classifier will only be inconsistent with the true
classification rule in a small margin around each of the true linear
separators. Then the error of the classification rule is easily
bounded by bounding the probability mass of these margins.

To measure the approximation quality, we say that the half-ball
$\halfball = \halfball(\hat x, r, \hat w)$ is an
$\alpha$-approximation to the linear function $h$ if
$\prob_{X \sim \halfball}(\sign(h(X)) = \sign(h(\hat x))) \leq
\alpha$, where $\prob_{X \sim \halfball}$ denotes the probability when
$X$ is sampled uniformly from the half-ball $\halfball$. The
motivation for this definition is as follows: given any point
$\hat x \in \X$, the half-ball $\halfball(\hat x, r, \hat w)$ will be
an $\alpha$-approximation to $h_i$ only if $\hat x$ is on one side of
the decision surface of $h_i$ and all but an $\alpha$-fraction of the
half-ball's volume is on the other side. Intuitively, this means that
the face of the half-ball must approximate the decision surface of the
function $h_i$.

The following Lemma shows that when Algorithm~\ref{alg:edgedetection}
is run with appropriate parameters and on a large enough sample drawn
from the data distribution, then with high probability the algorithm
will include at least one half-ball in $H$ $\alpha$-approximating each
true hyperplane $h_i$ and every half-ball in $H$ will be an
$\alpha$-approximation to at least one true hyperplane. Recall that
$\hbp(r) = \frac{1}{2} \clb r^d v_d$ is a lower bound on the
probability mass of a half-ball of radius $r$ contained in the set
$K$.

\begin{restatable}{lem}{HyperplaneApproximationsLemma}
  \label{lem:hyperplaneapproximations}
  Fix any $\alpha > 0$ and confidence parameter $\delta > 0$. Let $H$
  be the set of half-balls produced by
  Algorithm~\ref{alg:edgedetection} when run with parameters $r = R/2$
  and $\tau = \frac{1}{2}\alpha \hbp(r)$ on a sample of size
  $n = O(\frac{1}{\gamma^2}(\ln^2 \frac{d}{\gamma} + \ln
  \frac{1}{\delta}))$ where
  $\gamma = \frac{2}{5}\tau = \frac{1}{5}\alpha \hbp(r)$. Then with
  probability at least $1-\delta$, every half-ball in $H$ will be an
  $\alpha$-approximation to at least one true hyperplane $h_i$, and
  every true hyperplane $h_i$ will be $\alpha$-approximated by at
  least one half-ball in $H$.
\end{restatable}
\begin{proof}
  Since the VC-dimension of both balls and half-spaces in $\reals^d$
  is $d+1$, the VC-dimension of the set of intersections of balls and
  up to two half-spaces is $O(d \ln d)$. Therefore, by a standard
  VC-bound~\citep{VC}, if we see an iid sample $\sampleset$ of size
  $n = O(\frac{1}{\gamma^2} (\ln^2 \frac{d}{\gamma} + \ln
  \frac{1}{\delta}))$, then with probability at least $1-\delta$ the
  empirical measure of any ball intersected with up to two half-spaces
  will be within $\gamma$ of its true probability mass. In other
  words, the fraction of the sample set $\sampleset$ that lands in any
  ball intersected with up to two half-spaces will be within $\gamma$
  of the probability that a sample $X$ drawn from $\datadist$ will
  land in the same set. For the remainder of the proof, assume that
  this high-probabilty event holds.

  First, we show that every half-ball in the set $H$ is an
  $\alpha$-approximation to at least one true hyperplane. Suppose
  otherwise, then there is a half-ball
  $\halfball = \halfball(\hat x, r, \hat w)$ with
  $(\hat x, \hat w) \in H$ that is not an $\alpha$ approximation to
  any true hyperplane $h_i$. The center $\hat x$ of the half-ball must
  belong to the positive region $\K$, since it is one of the sample
  points. If the half-ball $\halfball$ is contained entirely in the
  set $\K$, then the probability that a new sample $X$ drawn from
  $\datadist$ will land in the half-ball $\halfball$ is $\hbp(r)$ and
  therefore the fraction of samples that landed in the half-ball is at
  least $\hbp(R/2) - \gamma$. But since
  $\hbp(r) - \gamma \geq \frac{4}{5}\alpha\hbp(r) > \tau$, this
  contradicts the half-ball being included in the set $H$. Otherwise,
  the half-ball contains at least one point $y$ that does not belong
  to the set $\K$ (i.e., it does not belong to any class). Since
  $\hat x$ is in the set $\K$, there is at least one true hyperplane
  $h_i$ that separates $\hat x$ from $y$. Since $r = R/2 < R$, every
  other point $y'$ in the half-ball that does not belong to any class
  must have the same code word as $y$ (since, by assumption, points
  outside of $K$ that do not belong to any class must have the same
  code word when they are closer than $R$), and therefore must be on
  the same side of $h_i$ as $y$. It follows that all points in the
  half-ball on the same side of $h_i$ as $\hat x$ (i.e., those points
  for which the sign of $h_i$ matches the sign of $h_i(\hat x)$)
  belong to the set $\K$. But, since the half-ball is not an
  $\alpha$-approximation to $h_i$, this implies that at least an
  $\alpha$ fraction of the half-ball's volume must belong to the set
  $\K$. Therefore, the probability that a new sample $x$ drawn from
  the data distribution $p$ belongs to the half-ball can be lower
  bounded as follows:
  \[
    \prob_{x \sim p}(x \in \halfball)
    \geq \clb \vol(\halfball \cap \K) 
    = \clb \vol(\halfball)\frac{\vol(\halfball \cap K)}{\vol(\halfball)} 
    \geq \alpha 
    \hbp(r).
  \]
  By the uniform convergence argument, the fraction of the samples in
  $\sampleset$ contained in the half-ball $\halfball$ is at least
  $\alpha \hbp(r) - \gamma > \tau$, which contradicts the half-ball
  being in $H$. In either case we arrived at a contradiction and it
  follows that every half-ball in $H$ is an $\alpha$-approximation to
  at least one true hyperplane $h_i$.

  Finally, we show that the set $H$ will contain at least one
  half-ball that is an $\alpha$-approximation to each true hyperplane
  $h_i$. Fix any true hyperplane $h_i$. By assumption, there is a
  class $\ell$ and a point $x_0$ on the decision surface of $h_i$ so
  that one half-ball of $\ball(x_0, R)$ with face $h_i$ is is
  contained in $\K_\ell \subset \K$ and the other half-ball is
  disjoint from $\K$.  Suppose WLOG that the half-ball on the negative
  side of $h_i$ is contained in $\K$ (the case when the half-ball on
  the positive side is contained in $\K$ is identical). Define
  $\rho > 0$ to be the width such that the probability that a new
  sample $X$ from $\datadist$ lands in the slice of the ball
  $S = \setc{x \in \ball(\hat x_0, r)}{h_i(x) \in [-\rho,0]}$ is equal
  to $\tau - \gamma$. Note that, since the half-ball on the negative
  side of $h_i$ is a subset of $K$ and
  $\tau - \gamma = \frac{3}{10}\alpha \hbp(r) < \hbp(r)$, such a value
  of $\rho$ always exists. Since $\tau - \gamma > \gamma$, the uniform
  convergence argument guarantees that there will be at least one
  sample point in the slice, say $\hat x \in \sampleset$. Since
  $\hat x$ is within distance $r = R/2$ of the point $x_0$, the ball
  $\ball(\hat x, r)$ is contained in $\ball(x_0, R)$. Therefore, the
  ball of radius $r$ centered at $\hat x$ only contains points that
  either belong to class $\ell$ or no class, since only the linear
  separator $h_i$ passes through this ball. By construction, the
  half-ball $\halfball(\hat x, r, w_i)$ (where $w_i$ is the
  coefficient vector defining $h_i(x) = w_i^\tp x - b_i$) with face
  parallel to $h_i$ intersects the set $\K$ in a slice of width at
  most $\rho$ and therefore has probability mass at at most
  $\tau - \gamma$. It follows that the direction $\hat w$ that
  minimizes the number of samples in the half-ball
  $\halfball(\hat x, r, \hat w)$ will result in the half-ball
  containing at most a $\tau$ fraction of the sample set, and
  therefore the pair $(\hat x, \hat w)$ will be included in $H$, and
  this will be an $\alpha$-approximation to $h_i$.
\end{proof}

Naturally, if a half-ball $\halfball(\hat x, r, \hat w)$ is an
$\alpha$-approximation to the linear function $h$, we expect that the
decision surface of $\hat h(x) = \hat w^\tp(x - \hat x)$ is similar to
the decision surface of $h$. In turn, this suggests that either
$\hat h(x)$ or $-\hat h(x)$ should take similar function values to
$h(x)$ (since the coefficient vectors are normalized). We first give a
simple probability lemma that bounds the fraction of a ball contained
between two parallel hyperplanes, one passing through the ball's
center. The proof of Lemma~\ref{lem:ballslice} is in
Section~\ref{apdx:boundaryfeatures} of the appendix.

\begin{restatable}{lem}{lemBallSlice}
  \label{lem:ballslice}
  Let $r > 0$ be any radius and $X$ be a random sample drawn uniformly
  from the ball of radius $r$ centered at the origin. For any width
  $0 \leq \rho \leq r / \sqrt{2}$, the probability that the first
  coordinate of $X$ lands in $[0,\rho]$ can be bounded as follows:
  \[
  \sqrt{\frac{d}{2^d \pi}} \frac{\rho}{r}
  \leq \prob_{X \sim B(r,0)}(X_1 \in [0,\rho])
  \leq \sqrt{\frac{d+1}{2\pi}} \frac{\rho}{r}.
  \]
\end{restatable}

Using Lemma~\ref{lem:ballslice}, we show the following:
\begin{restatable}{lem}{lemProbToFunc}
  \label{lem:probtofunc}
  Let the half-ball $\halfball(\hat x, r, \hat w)$ be an
  $\alpha$-approximation to the linear function $h(x) = w^\tp x - b$
  with $\norm{w} = 1$, $\hat x \in \X$, and $\alpha <
  \frac{1}{2}$. Let $D$ be the diameter of $\X$. If $h(\hat x) < 0$
  then for all $x \in \X$ we have
  \[
  |h(x) - \hat h(x)| \leq \biggl(2D + \sqrt{\frac{2^d \pi}{d}} \frac{r}{2}\biggr) \sqrt{\alpha},
  \]
  where $\hat h(x) = \hat w^\tp (x - \hat x)$. Otherwise, if
  $h(\hat x) > 0$ then the same upper bound holds for
  $|h(x) + \hat h(x)|$.
\end{restatable}

\begin{proof}
  Suppose that $h(\hat x) < 0$ and let $X$ be a uniformly random
  sample from the half-ball
  $\halfball = \halfball(\hat x, r, \hat w)$. By assumption, we know
  that $\prob(h(X) < 0) \leq \alpha$.

  First we show that $\norm{w - \hat w}$ is small. Since
  $\alpha < 1/2$ we have that $w^\tp \hat w > 0$. To see this, notice
  that we must have $h(\hat x + r \hat w) \geq 0$, since otherwise at
  least half of the half-ball would be on the negative side of
  $h$. Define $g(x) = w^\tp(x - \hat x)$ to be the linear function
  whose decision surface runs parallel to that of $h$ but passes
  through the point $\hat x$. Since
  $h(x) = g(x) + h(\hat x) \leq g(x)$, we have that
  $\alpha > \prob(h(X) < 0) \geq \prob(g(X) < 0)$. Moreover, since the
  decision surface of $g$ passes through the center of the half-ball
  $\halfball$ and the uniform distribution on the half-ball is
  radially symmetric about the point $\hat x$, we have that
  $\prob(g(X) < 0) = \frac{\ang(w, \hat w)}{\pi}$.  It follows that
  $\ang(w, \hat w) \leq \pi \alpha$. Using this fact, we can bound
  $\norm{w - \hat w}$ as follows:
  \[ \norm{w - \hat w}^2 = \norm{w}^2 + \norm{\hat w}^2 - 2w^\tp \hat
    w = 2(1 - w^\tp \hat w). \] Since
  $w^\tp \hat w = \cos(\ang(w, \hat w))$ and on the interval
  $[0,\pi/2]$, the $\cos(\theta)$ function is decreasing and lower
  bounded by $1 - \frac{2}{\pi}\theta$, we have that
  $2(1-w^\tp \hat w) \leq 4\alpha$. Taking the square root gives that
  $\norm{w - \hat w} \leq 2\sqrt{\alpha}$.

  Next we show that $|h(\hat x)|$ (the distance from $\hat x$ to the
  decision surface of $h$) is not too large. The half-ball
  $\halfball(\hat x, r, w)$, whose direction $w$ matches the
  coefficient vector of $h$ is one half-ball centered at $\hat x$ of
  radius $r$ minimizing the fraction of its volume contained on the
  same side of $h$ as $\hat x$. This is because every point in the
  ball $\ball(\hat x, r)$ not on the same side as $\hat x$ is
  contained in $\halfball(\hat x, r, w)$. Let $Y$ be uniformly sampled
  from $\halfball(\hat x, r, w)$. By construction of the half-ball $Y$
  is sampled from, we have that
  $\prob(h(X) < 0) \geq \prob(h(Y) < 0)$, which gives
  \[
    \alpha
    \geq \prob_{X \sim \halfball(\hat x, r, \hat  w)}\bigl(h(X) < 0\bigr) 
    \geq \prob_{Y \sim \halfball(\hat x, r, w)} \bigl(h(Y) < 0\bigr) 
     \geq \sqrt{\frac{d}{2^d\pi}} \frac{2|h(\hat x)|}{r},
  \]
  which implies that
  \[
  |h(\hat x)|
  \leq \sqrt{\frac{2^d \pi}{d}} \frac{r \alpha}{2}.
  \]

  Finally, let $x'$ be any point on the decision surface of $h$, so
  that $h(x) = w^\tp(x - x')$. Combining the above calculations we
  have
  \begin{align*}
    |h(x) - \hat h(x)|
    &= |w^\tp (x - x') - \hat w^\tp(x - \hat x)| \\
    &= |w^\tp (x - \hat x) + w^\tp(\hat x - x') - \hat w^\tp(x - \hat x)| \\
    &= |(w - \hat w)(x - \hat x) + w^\tp(\hat x - x')| \\
    &\leq \norm{w - \hat w}\norm{x - \hat x} + |h(\hat x)| \\
    &\leq 2\sqrt{\alpha}D + \sqrt{\frac{2^d \pi}{d}} \frac{r\alpha}{2} \\
    &\leq (2D + \sqrt{\frac{2^d \pi}{d}} \frac{r}{2}) \sqrt{\alpha},
  \end{align*}
  as required. The proof of the case when $h(x) > 0$ follows by
  applying the above arguments to the function $-h$.
\end{proof}

Recall that for any hyperplane $h(x) = w^\tp x - b$ with
$\norm{w}_2 = 1$, the distance from point $x$ to the decision surface
of $h$ is $|h(x)|$. The above lemma implies that if
$\halfball(\hat x, r, \hat w)$ is an $\alpha$-approximation to $h$,
then either $\hat h$ or $-\hat h$ will have the same sign as $h$ for
all points in $\X$ except those in a margin of width
$O(\sqrt{\alpha})$ around $h$. Under the uniform distribution on $\K$,
the probability mass of the margins surrounding the true hyperplanes
isn't large, which results in low error for the classification rule.

\begin{restatable}{thm}{EdgeDetectionTheorem}
  \label{thm:edgedetection}
  Suppose Assumptions~\ref{assm:nearlyuniform} and
  \ref{assm:boundaryfeatures} hold. For any desired error
  $\epsilon > 0$, with probability at least $1-\delta$, running
  Algorithm~\ref{alg:edgedetection} with parameters $r \leq R/2$ and
  $\tau = \alpha\hbp(r)/2$ for a known constant $\alpha$ on on a
  sample of size
  $n = \tilde O( d m^2 \cub^2 R^d / (\clb^2 \epsilon^4) )$ will have
  error at most $\epsilon$.
\end{restatable}

\begin{proof}
  By Lemma~\ref{lem:hyperplaneapproximations}, for the parameter
  settings $\tau$ and $r$ and the given sample size, with probability
  at least $1-\delta$ every half-ball included in the set $H$ will be
  an $\alpha$-approximation to some true hyperplane $h_i$, and every
  true hyperplane $h_i$ is $\alpha$-approximated by at least one
  half-ball in $H$. Assume that this high probability event occurs.

  Let $H = \{(\hat x_1, \hat w_1), \dots, (\hat x_M, \hat w_M)\}$ be
  the set of of half-balls produced by the algorithm and define
  $\hat h_i(x) = \hat w_i^\tp (x - \hat x_i)$ for $i = 1, \dots, M$ to
  be the corresponding linear
  functions. Algorithm~\ref{alg:edgedetection} uses these hyperplanes
  to partition the space $\X$ into a collection of polygonal regions
  and assigns a unique class label to each cell in the
  partition. Notice that negating any of the $\hat h_i$ functions does
  not change the partitioning of the space.  Therefore, negating any
  subset of the $\hat h_i$ will not change the permutation-invariant
  error of the resulting classifier.

  Let $I_1$, \dots, $I_m$ be a partition of the set of indices
  $\{1, \dots, M\}$ such that for all $j \in I_i$, we have that
  $\halfball(\hat x_j, \hat w_j, r)$ is an $\alpha$-approximation to
  $h_i$. By Lemma~\ref{lem:probtofunc}, we know that for at least one
  $g \in \{\hat h_j, -\hat h_j\}$, we have that
  \[
  |h_i(x) - g(x)|
  \leq \left( 2D + \sqrt{\frac{2^d\pi}{d}} \frac{r}{2}\right) \sqrt{\alpha}
  \]  
  Since negating any of the functions $\hat h_j$ does not change the
  error of the resulting classifier, assume WLOG that the above holds
  for $g = \hat h_j$.

  This implies that whenever $|h_i(x)| > c\sqrt{\alpha}$, where
  $c = 2D + \sqrt{\frac{2^d\pi}{d} \frac{R}{4}}$, then for every
  $j \in I_i$, the sign of $\hat h_j(x)$ is the same as the sign of
  $h_i(x)$. It follows that for points that are not within a margin of
  $c\sqrt{\alpha}$ of any of the true hyperplanes, every $\hat h_j$
  function with $j \in I_i$ will have the same sign as $h_i$ for all
  $i = 1, \dots, m$. It follows that the classifier can only error on
  points that are within a $c \sqrt{\alpha}$ margin of one of the true
  hyperplanes.

  Using Lemma~\ref{lem:ballslice} we can bound the probability that a
  sample $X$ drawn uniformly from $K$ lands in the
  $c\sqrt{\alpha}$-margin of hyperplane $h_i$ as follows:
  \[
  \prob(\hbox{$X$ in $c\sqrt{\alpha}$-margin of $h_i$})
  \leq 2 \sqrt{\frac{d+1}{2\pi}} \frac{c \sqrt{\alpha}}{D} D^d v_d \cub,
  \]
  where $D$ is the diameter of $X$. We can make this upper bound equal
  to $\epsilon / m$ by setting
  \[
  \alpha
  = \frac{\pi}{2(d+1)} \left( \frac{\epsilon D}{m c D^d v_d \clb}\right)^2
  = \Omega\left(\frac{\epsilon^2}{m^2 2^d R^2 D^{2d} v_d^2 \clb^2}\right)
  \]
  Applying the union bound to the $m$ hyperplanes $h_1$, \dots, $h_m$
  shows that the error of $\hat f$ is at most $\epsilon$.
\end{proof}

Note that if the scale parameter $R$ is unknown, the conclusions of
Theorem~\ref{thm:edgedetection} continue to hold when the parameter
$r$ and the unlabeled sample complexity $n$ are set using a
conservatively small estimate $\hat R$ satisfying $\hat R \leq R$.

Theorem~\ref{thm:edgedetectionAgnostic} in the appendix extends the
above result to the agnostic setting considered in
Section~\ref{sec:agnostic}.

%% file: agnostic.tex
\section{Extensions to the Agnostic Setting}
\label{sec:agnostic}

The majority of our algorithms have two phases: first, we extract a
partitioning of the unlabeled data into groups that are likely
label-homogeneous, and second, we query the label of the largest
groups. We can extend our results for these algorithms to the agnostic
setting by querying multiple labels from each group and using the
majority label.

Specifically, suppose that the data is generated according to a
distribution $P$ over $\X \times [L]$ and there exists a labeling
function $f^*$ such that
$\prob_{(x,y) \sim P}(f^*(x) \neq y) \leq \eta$ and our assumptions
hold when the unlabeled data is drawn from the marginal $P_\X$ but the
labels are assigned by $f^*$. That is, the true distribution over
class labels disagrees with a function $f^*$ satisfying our
assumptions with probability at most $\eta$. In this setting, the
first phase of our algorithms, which deals with only unlabeled data,
behaves exactly as in the realizable setting. The only difference is
that we will need to query multiple labels from each group of data to
ensure that the majority label is the label predicted by
$f^*$. Suppose that the training data is $(x_1, y_1)$, \dots,
$(x_n, y_n)$ drawn from $P$ (where the labels $y_i$ are initially
unobserved). For $n = \tilde O(1 / \eta^2)$, we are guaranteed that on
at most $2\eta n$ of the training points we have that
$y_i \neq f^*(x_i)$. Moreover, if we only need to guess the label of
large groups of samples, say those containing at least $8\eta n$
points, then we are guaranteed that within each group at least $1/4$
of the sample points will have labels that agree with
$f^*$. Therefore, after querying $O(\log(1/\delta))$ labeled examples
from each group, the majority label will agree with $f^*$. If we use
these labels in the second phase of the algorithm, we would be
guaranteed that the error of our algorithm would be at most $\epsilon$
had the labels been produced by $f^*$, and therefore the error under
the distribution $P$ is at most $\eta + \epsilon$. The appendix
contains agnostic versions of Theorems~\ref{thm:hamDPO},
\ref{thm:onevsall}, and \ref{thm:edgedetection}.

Similarly, modifying Algorithm~\ref{alg:activehierarchicallinkage} to
require that the each cluster in the pruning have a majority label
that accounts for at least $3/4$ of the cluster's data can be used to
extend the corresponding results to the agnostic setting.

%% file: discussion.tex
\section{Conclusion and Discussion}
\label{sec:conclusion}

In this work we showed how to exploit the implicit geometric
assumptions made by output code techniques under the well studied
cases of one-vs-all and well separated codewords, and for a novel
boundary features condition that captures the intuition that every
binary learning task should be significant. We provide label-efficient
learning algorithms for both the consistent and agnostic learning
settings with guarantees when the data density has thick level sets or
upper and lower bounds. In all cases, our algorithms show that the
implicit assumptions of output code learning can be used to learn from
very limited labeled data.

In this work we focused on linear output codes, which have been in
several practical works. For example \citet{zero-shotpmhm09} use
linear output codes for neural decoding of thoughts from fMRI data,
\citet{Berger1999} used them successfully for text classification, and
\citet{Crammer2000} show that they perform well on MNIST and several
UCI datasets. Many other works use non-linear output codes, and it is
a very interesting research direction to extend our work to such
cases.

The unlabeled sample complexity of our algorithms is exponential in
the dimension because our algorithms require the samples to cover
high-density regions. It is common for semi-supervised algorithms to
require exponentially more unlabeled data than labeled,
e.g. \citep{Singh2008, Castelli1995}. Our results also show that the
unlabeled sample complexity only scales exponentially with the
intrinsic dimension, which may be significantly lower than the ambient
dimension for real-world problems. An interesting direction for future
work is to determine further conditions under which the unlabeled
sample complexity can be drastically reduced.

%% file: acknowledgements.tex
\subsection*{Acknowledgments}
This work was supported in part by NSF grants CCF-1422910,
CCF-1535967, IIS-1618714, a Sloan Research Fellowship, a Microsoft
Research Faculty Fellowship, and a Google Research Award.

%% file: hamming_dp1_appendix.tex
\section{Appendix for Error Correcting Output Codes}
\label{apdx:hamdp1}

First, we show that the line segment $[x,y]$ crosses the decision
surface of the linear separator $h_k$ if and only if $h(x)$ and $h(y)$
differ on the $k^{\rm th}$ entry.

\begin{lem}
  \label{lem:ham2intersections}
  Let $i \neq j$ be any pair of classes whose codewords disagree on
  the $k^{\rm th}$ bit. Then for any points $x \in \K_i$ and
  $y \in \K_j$, the line segment $[x,y]$ intersects with the line
  $h_k = 0$.
\end{lem}
\begin{proof}
  Without loss of generality, suppose that $\C_{ik} = 1$ and
  $\C_{jk} = -1$.  Then, from the definition of $\K_i$ and $\K_j$, we
  have that $h_k(x) > 0$ and $h_k(y) < 0$. The function
  $f(t) = h_k((1-t)x+ty)$ is continuous and satisfies
  $f(0) = h_k(x) > 0$ and $f(1) = h_k(y) < 0$. It follows that there
  must be some $t_0 \in (0,1)$ such that $f(t_0) = 0$. But this
  implies that the point $z = (1-t_0)x + t_0y \in [x,y]$ satisfies
  $h_k(z) = 0$ and it follows that $h_k = 0$ intersects with $[x,y]$
  at the point $z$.
\end{proof}

Next, we show that when the consistent linear output code makes at
most $\beta$ errors when predicting the code word of a new example and
the Hamming distance of the code words is at least $2\beta + d + 1$,
then there must be a minimum gap $g > 0$ between any pair of points
belonging to different classes.

\SeparationLemma*
\begin{proof}
  For sets $A$ and $B$, let
  $d(A,B) = \min_{a \in A, b \in B} \norm{a - b}$ denote the distance
  between them and recall that for each $i = 1, \dots, L$, we defined
  $K_i = \setc{x \in \X}{\hamd(h(x), C_i) \leq \beta}$ to be the set
  of points that belong to class $i$.

  Fix any pair of classes $i$ and $j$ and suppose for contradiction
  that $\dist(\K_i, \K_j) = 0$. This implies that there are two code
  words $c, c' \in \{\pm 1\}^m$ such that $\hamd(c, C_i) \leq \beta$,
  $\hamd(c', C_j) \leq \beta$, and the distance between
  $A = \setc{x \in \X}{h(x) = c}$ and $B = \setc{x \in \X}{h(x) = c'}$
  is 0.  First, we construct a point $x$ that belongs to
  $\overline A \cap \overline B$, where $\overline A$ and
  $\overline B$ denote the closure of $A$ and $B$, respectively. Since
  $\dist(A, B) = 0$, there exists a sequence of points
  $x_1, x_2, \ldots \in A$ such that $d(x_n, B) \to 0$ as
  $n \to \infty$. But, since $A$ is bounded, so is the sequence
  $(x_n)$, and therefore by the Bolzano-Weierstrass theorem, $(x_n)$
  has a convergent subsequence. Without loss of generality, suppose
  that $(x_n)$ itself converges to the point $x$. Then $x$ is a limit
  point of $A$ and therefore belongs to the closure of $A$. On the
  other hand, since the function $z \mapsto d(z, B)$ is continuous, it
  follows that $d(x,B) = \lim_{n \to \infty} d(x_n, B) = 0$ and
  therefore $x$ is also in the closure of $B$.

  Now let $k$ be any index such that the code words $c$ and $c'$
  differ on the $k^{\rm th}$ entry. Next, we show that $h_k(x) =
  0$. For each integer $n > 0$, let $C_n = B(x, 1/n)$ be the ball of
  radius $1/n$ centered at $x$. Since $x$ belongs to the closure of
  $A$ and $C_n$ is a neighborhood of $x$, we can find some point, say
  $x_n$ that belongs to the intersection $A \cap C_n$. Similarly, we
  can find a point $y_n$ belonging to $B \cap C_n$. Since the line
  segment $[x_n, y_n]$ passes from $A$ to $B$,
  Lemma~\ref{lem:ham2intersections} guarantees that there is a point
  $z_n \in [x_n, y_n] \subset C_n$ such that $h_k(z_n) = 0$. But, by
  construction, the sequence $z_n$ is converging to $x$ and, since
  linear functions are continuous, it follows that
  $h_k(x) =\lim_{n \to \infty} h_k(z_n) = 0$.

  But this leads to a contradiction: since the codewords $c$ and $c'$
  must disagree on at least $d+1$ entries, at least $d+1$ of the
  linear separators $h_1$, \dots, $h_\numseparators$ intersect at the
  point $x$, which contradicts our assumption that at most $d$ lines
  intersect at any point $x \in \X$.  Therefore, we must have
  $\dist(\K_i, \K_j) > 0$. Since there are finitely many classes,
  taking $g = \min_{i,j} \dist(\K_i, \K_j)$ completes the proof.
\end{proof}

Next, we prove a similar result to Theorem~\ref{thm:hamDPO} that holds
in the agnostic setting of Section~\ref{sec:agnostic}.

\begin{restatable}{thm}{HamDPOAgnosticTheorem}
  \label{thm:hamDPOAgnostic}
  Assume Assumption~\ref{assm:ecoc}, $\error(f^*) \leq \eta$, and $p$
  has $C$-thick level sets. For $0 < \epsilon \leq \eta$, suppose
  $\{p \geq \epsilon/(2\vol(K))\}$ has $N$ connected components, each
  with probability at least $7\eta$. With probability at least
  $1-\delta$, running Algorithm~\ref{alg:singlelinkage} with parameter
  $r_c < g$ on an unlabeled sample of size
  $n = \tilde O(\frac{1}{\epsilon^2}((4C)^{2d}d^{d+1}/r_c^2d + N))$
  and querying $t = O(\ln N/\delta)$ labels per cluster will have
  error at most $\eta + \epsilon$ after querying at most $Nt$ labels.
\end{restatable}
\begin{proof}
  Define $\lambda = \epsilon / (2\vol(K))$ and let $A_1$, \dots, $A_N$
  be the connected components of $\{p \geq \lambda\}$. Since
  Assumption~\ref{assm:ecoc} holds, Lemma~\ref{lem:separation}
  guarantees that there is a distance $g > 0$ such that whenever
  $f^*(x) \neq f^*(x')$, we must have $\norm{x - x'} \geq g$. This
  implies that for any $\lambda' > 0$, $f^*$ must be constant on the
  connected components of $\{p \geq \lambda'\}$, since otherwise we
  could construct a pair of points closer than $g$ with
  $f^*(x) \neq f^*(x')$. In particular, we know that $f^*$ is constant
  on each of the $A_i$ sets.

  Since the clustering produced by Algorithm~\ref{alg:singlelinkage}
  does not see the labeled examples, an identical covering argument to
  the one in the proof of Theorem~\ref{thm:hamDPO} shows that for
  $n = O((4C)^{2d} d^{d+1} / (\epsilon^2 r_c^{2d}))$ with probability
  at least $1-\delta$, for each set $A_i$ there is a unique cluster,
  say $\hat A_i$, such that $\hat A_i$ contains $\sampleset \cap A_i$,
  the closest cluster to every point in $A_i$ is $\hat A_i$. Assume
  this high probability event occurs.

  Similarly to the proof of Theorem~\ref{thm:hamDPO}, for
  $n = O(\frac{N}{\epsilon^2} \ln \frac{1}{\delta})$, we have that
  with probability at least $1-\delta$, for any subset of indices
  $I \subset [N]$, we have that
  \[
    \bigl| |\sampleset \cap A_I|/n - P_\X(A_I) \bigr| \leq \epsilon,
  \]
  where $A_I = \bigcup_{i \in I} A_i$. Assume this high probability
  event occurs.

  Now let $y_1$, \dots, $y_n$ be the (unobserved) labels corresponding
  to the unlabeled sample $x_1$, \dots, $x_n$. Since
  $\prob_{(x,y)\sim P}(f^*(x) \neq y) \leq \eta$, if
  $n = O(\frac{1}{\eta^2}\ln \frac{1}{\delta})$, then with probability
  at least $1-\delta$, we have that $f^*(x_i) \neq y_i$ for at most
  $2\eta n$ of the sample points.

  Now, for any connected component $A_i$, let $\hat A_i$ be the
  cluster containing $A_i \cap \sampleset$. Since we have uniform
  convergence for all unions of the $A_i$ sets, and
  $P_\X(A_i) \geq 7\eta$, we know that the set $A_i$ contains at least
  $6\eta n$ sample points. Therefore, even if every point whose label
  $y_i$ disagrees with $f^*$ belongs to $\hat A_i$, we know that at
  most a $2\eta n / (6 \eta n) = 1/3$ fraction of the points belonging
  to the cluster $\hat A_i$ will have labels other than $f^*(A_i)$. If
  we query the label of
  $t = 32 \ln \frac{2N}{\delta} = O(\ln \frac{N}{\delta})$ points
  belonging to cluster $\hat A_i$, then with probability at least
  $1-\delta/N$ the majority label will agree with $f^*$ on
  $A_i$. Applying the union bound over the connected components $A_1$,
  \dots, $A_N$ gives the same guarantee for all connected components
  with probability at least $1-\delta$.

  Let $\hat f$ be the classifier output by
  Algorithm~\ref{alg:singlelinkage} and $Q \subset [N]$ be the indices
  of the $A_i$ sets for which the algorithm queried the label of the
  corresponding cluster $\hat A_i$. The above arguments show that with
  probability at least $1-4\delta$, we have that $\hat f(x) = f^*(x)$
  for any $x \in \cup_{i \in Q} A_i$ and, as in
  Theorem~\ref{thm:hamDPO}, we know that
  $P_\X(\bigcup_{i \not \in Q}A_i) \leq \epsilon/2$. This gives the
  following bound on the error of $\hat f$: Let $(x,y) \sim P$, then
  \begin{align*}
    \prob(\hat f(x) \neq y)
    &= \prob(\hat f(x) \neq y, x \in \{p < \lambda\}) \\
    &\quad + \prob(\hat f(x) \neq y, x \in \bigcup_{i \in Q} A_i) \\
    &\quad + \prob(\hat f(x) \neq y, x \in \bigcup_{i \not \in Q} A_i) \\
    &\leq \prob(x \in \{p < \lambda\}) \\
    &\quad + \prob(f^*(x) \neq y) \\
    &\quad + \prob(x \in \bigcup_{i \not \in Q} A_i).
  \end{align*}
  By our choice of $\lambda$, the first term is at most $\epsilon/2$,
  by assumption the second term is at most $\eta$, and the last term
  is at most $\epsilon/2$, giving the final error bound of
  $\eta + \epsilon$. 
\end{proof}

%% file: one_vs_all_appendix.tex
\section{Appendix For One-vs-all on the Unit Ball}

The following result is similar to Theorem~\ref{thm:onevsall} and
shows that Algorithm~\ref{alg:ova} continues to work in the agnostic
setting of Section~\ref{sec:agnostic}.

\begin{restatable}{thm}{onevsallTheoremAgnostic}
\label{thm:onevsallAgnostic}
Suppose the data is drawn from distribution $P$ over $\X \times [L]$
and that there exists a labeling function $f^*$ such that
$\prob_{(x,y) \sim P}(f^*(x) \neq y) \leq \eta$ and
Assumptions~\ref{assm:ova} and \ref{assm:nearlyuniform} hold when
labels are assigned by $f^*$. Assume that
$\prob_{x \sim P_\X}(f^*(x) = i) \geq 19\eta$ for all classes $i$. For
any excess error $\epsilon$, There exists an $r_c$ satisfying
$r_c = \Omega( \epsilon \clb / (\cub^2\bmin))$ such that with
probability at least $1-\delta$, running Algorithm~\ref{alg:ova} with
parameter $r_c$ on an unlabeled sample of size
$n = \tilde O((\cub^4d/(\epsilon^2 \clb^2 \bmin^2))^d)$ and querying
$t = O(\ln \frac{N}{\delta})$ labels from each cluster will output a
classifier with error at most $\eta + \epsilon$ and query at most $tL$
labels.
\end{restatable}
\begin{proof}
  For small enough $\epsilon$, we know that at least half of the
  probability mass of the points assigned to class $i$ will belong to
  the $\epsilon$-level set of $\{\qub{i} \geq \epsilon\}$ (in the
  notation of Theorem~\ref{thm:onevsall}). Therefore, the probability
  mass of each of the sets $A_1$, \dots, $A_L$ in the proof of
  Theorem~\ref{thm:onevsall} is at least $9\eta$. It follows that if
  we see an unlabeled set of size $n = \tilde O(\frac{1}{\eta^2})$,
  then with probability at least $1 - \delta$ every $A_i$ set will
  contain at least $8\eta n$ points. Since these points belong to
  $A_i$, we know that they will be active, included in the graph $G$,
  and connected to the cluster that contains samples belonging to
  $A_i$. Moreover, under the same high probability event, we know that
  there are at most $2\eta n$ points whose labels disagree with
  $f^*$. Therefore, the cluster that contains samples from $A_i$ must
  have at least $8\eta n$ points, at most $2\eta n$ of which can have
  labels that disagree with $f^*$, so the label assigned by $f^*$ will
  account for at least a $3/4$ fraction of the points belonging to the
  cluster containing $A_i$. It follows that if we query
  $O(\log(L/\delta))$ labels from each $A_i$ set then with probability
  at least $1-\delta$, we will output a classification rule that
  agrees with $f^*$ except with probability $\epsilon$. It follows
  that the error with respect to $P$ at most $\eta + \epsilon$.
\end{proof}

%% file: boundary_features_appendix.tex
\section{Appendix for Boundary Features Condition}
\label{apdx:boundaryfeatures}

We begin by proving the probability bounds for slices of a
$d$-dimensional ball under the uniform distribution.

\lemBallSlice*
\begin{proof}
Let $B$ be the ball of radius $r$ centered at the origin and $S = \setc{x \in
B}{x_1 \in [0,\rho]}$ be the slice of $B$ for which the first coordinate is in
the interval $[0, \rho]$. The probability that a uniformly random sample from
$B$ lands in the subset $S$ is given by $\vol(S) / \vol(B)$, where $\vol$
denotes the (Lebesgue) volume of a set.

We bound the volume of the set $S$ by writing the volume as a double integral
over the first coordinate $x_1$ and the remaining $d-1$ coordinates $x_R$.

\[
  \vol(S)
  = \int_0^\rho \int_{\reals^{d-1}} \ind{\norm{x_R}_2^2 \leq r^2 - x_1^2} \, d x_R \, d x_1
\]
Noticing that the inner integral is actually the volume of a $d-1$ dimensional
ball of radius $\sqrt{r^2 - x_1^2}$, and using the fact that for any $d$, the
volume of a $d$-dimensional ball of radius $r$ is $r^d v_d$, where $v_d$ is the
volume of the $d$-dimensional unit ball, we have
\[
\vol(S)
= v_{d-1} \int_0^\rho (r^2 - x_1^2)^{(d-1)/2} \, dx_1.
\]
Upper bounding the integrand by $r^{d-1}$ gives that $\vol(S) \leq v_{d-1} \rho
r^{d-1}$. Lower bounding the integrand by $(r^2 - \rho^2)^{(d-1)/2}$ and using
the fact that $\rho \leq \frac{r}{\sqrt 2}$ we have that $\vol(S) \geq v_{d-1}
\frac{1}{\sqrt{2^{d-1}}} \rho r^{d-1}$. Dividing both inequalities by the volume
of $B$, which is $r^d v_d$, and using the fact that for all $d$ we have
$\frac{v_{d-1}}{v_d} \in [\sqrt{\frac{d}{2\pi}}, \sqrt{\frac{d+1}{2\pi}}]$ gives
\[
\sqrt{\frac{d}{2^d\pi}} \frac{\rho}{r}
\leq \prob_{X \sim B}(X \in S)
\leq \sqrt{\frac{d+1}{2\pi}} \frac{\rho}{r},
\]
as required.
\end{proof}

The following is an extension of Theorem~\ref{thm:edgedetection} to
the agnostic setting described in Section~\ref{sec:agnostic}.

\begin{thm}
  \label{thm:edgedetectionAgnostic}
  Suppose the data is drawn from distribution $P$ over $\X \times [L]$
  and that there exists a labeling function $f^*$ such that
  $\prob_{(x,y) \sim P}(f^*(x) \neq y) \leq \eta$ and
  Assumptions~\ref{assm:nearlyuniform} and \ref{assm:boundaryfeatures}
  hold when labels are assigned by $f^*$. Moreover, assume that
  $\prob_{x \sim P_\X}(f^*(x) = i) \geq 10 \eta$ for all classes
  $i$. For any excess error $0 < \epsilon \leq \eta$, with probability
  at least $1-\delta$, running Algorithm~\ref{alg:edgedetection} with
  parameters $r \leq R/2$ and $\tau = \alpha\hbp(r)/2$ for a known
  constant $\alpha$ on on a sample of size
  $n = \tilde O( d m^2 \cub^2 R^d / (\clb^2 \epsilon^4) )$ and
  querying $t = O(\ln(N/\delta))$ labels from the $L$ largest clusters
  will have error at most $\eta + \epsilon$.
\end{thm}

\begin{proof}
  In the proof of Theorem~\ref{thm:edgedetection} we argued that with
  the set of hyperplanes produced by Algorithm~\ref{alg:edgedetection}
  will be good approximations to the true hyperplanes. We additionally
  showed that the set of hyperplanes approximating one of the linear
  separators $h_i$ defining the output code will agree with high
  probability with $h_i$ except in a small margin and we bounded the
  probability mass of these margins around each $h_i$ by
  $\epsilon$. It follows that for each class $i$, the probability mass
  of the set of points in that class not contained in these margins is
  at least $10\eta - \epsilon \geq 9\eta$, and it follows that if our
  unlabeled sample is of size at least $\tilde O(\frac{1}{\eta^2})$
  that with probability at least $1-\delta$, we will see at least
  $8\eta n$ points from each class which are not contained in the
  small margins. Under the same high probability event, we know that
  at most $2\eta n$ of the labels we query can disagree with $f^*$,
  which implies that the majority label within the $L$ largest cells
  will be the label predicted by $f^*$ on these cells. It follows that
  if we query the labels of $O(\ln N/\delta)$ labels from each class,
  then with probability at least $1-\delta$ the resulting classifier
  will predict labels that disagree with $f^*$ with probability at
  most $\epsilon$. It follows that the error of the classifier with
  respect to the distribution $P$ is at most $\eta + \epsilon$.
\end{proof}